\documentclass[letterpaper]{article} 
\usepackage{aaai24}  
\usepackage{times}  
\usepackage{helvet}  
\usepackage{courier}  
\usepackage[hyphens]{url}  
\usepackage{graphicx} 
\usepackage{natbib}  
\usepackage{caption} 
\usepackage{algorithm}

\usepackage[utf8]{inputenc} 
\usepackage[T1]{fontenc}    
\usepackage{url}            
\usepackage{booktabs}       
\usepackage{amsfonts}       
\usepackage{nicefrac}       
\usepackage{microtype}      
\usepackage{xcolor}         
\usepackage{newfloat}
\usepackage{listings}

\DeclareCaptionStyle{ruled}{labelfont=normalfont,labelsep=colon,strut=off} 
\floatstyle{ruled}
\newfloat{listing}{tb}{lst}{}
\floatname{listing}{Listing}
\title{From Past to Future: Rethinking  Eligibility Traces}
\author{
    Dhawal Gupta\textsuperscript{\rm 1}\footnote{Corresponding author},
    Scott M. Jordan\textsuperscript{\rm 2},
    Shreyas Chaudhari\textsuperscript{\rm 1}, \\
    Bo Liu\textsuperscript{\rm 3},
    Philip S. Thomas\textsuperscript{\rm 1}, 
    Bruno Castro da Silva\textsuperscript{\rm 1}
}
\affiliations{
    \textsuperscript{\rm 1}University of Massachusetts
    \textsuperscript{\rm 2}University of Alberta
    \textsuperscript{\rm 3}Amazon\\
    \{dgupta, schaudhari, pthomas, bsilva\}@cs.umass.edu, sjordan@ualberta.ca, boliucs@amazon.com

}

\usepackage{bibentry}

\usepackage{amsmath}
\usepackage{amsthm}
\usepackage{amssymb}
\usepackage{mathtools}
\usepackage{mathrsfs}
\allowdisplaybreaks
\mathtoolsset{showonlyrefs}

\theoremstyle{plain}
\newtheorem{theorem}{Theorem}[section]

\newtheorem{lemma}[theorem]{Lemma}

\theoremstyle{definition}

\newtheorem{assumption}[theorem]{Assumption}
\theoremstyle{remark}

\usepackage{courier} 

\usepackage{lipsum} 

\usepackage{graphicx}
\usepackage{subcaption}

\usepackage{theoremref}
\usepackage{graphicx}
\usepackage{xcolor}
\usepackage{cancel}
\definecolor{dark-red}{rgb}{0.4,0.15,0.15}
\definecolor{dark-blue}{rgb}{0,0,0.7}

\usepackage{algpseudocode}
\usepackage[vlined,linesnumbered,ruled,algo2e]{algorithm2e}
\SetKwProg{Fn}{Function}{}{}
\let\oldnl\nl
\newcommand{\nonl}{\renewcommand{\nl}{\let\nl\oldnl}}

\usepackage{multicol}
\usepackage{enumitem}
\usepackage[utf8]{inputenc} 
\usepackage[T1]{fontenc}    
\usepackage{url}            
\usepackage{booktabs}       
\usepackage{nicefrac}       
\usepackage{amsfonts}       
\usepackage{microtype}      
\usepackage{xcolor}         
\usepackage[bottom]{footmisc}

\usepackage[export]{adjustbox}

\makeatletter

\newcommand{\overleftrightsmallarrow}{\mathpalette{\overarrowsmall@\leftrightarrowfill@}}
\newcommand{\overrightsmallarrow}{\mathpalette{\overarrowsmall@\rightarrowfill@}}
\newcommand{\overleftsmallarrow}{\mathpalette{\overarrowsmall@\leftarrowfill@}}
\newcommand{\overarrowsmall@}[3]{%
  \vbox{%
    \ialign{%
      ##\crcr
      #1{\smaller@style{#2}}\crcr
      \noalign{\nointerlineskip \vskip 1pt}
      $\m@th\hfil#2#3\hfil$\crcr
    }%
  }%
}
\def\smaller@style#1{%
  \ifx#1\displaystyle\scriptscriptstyle\else
    \ifx#1\textstyle\scriptscriptstyle\else
      \scriptscriptstyle
    \fi
  \fi
}
\makeatother

\newcommand{\RV}{\overleftsmallarrow{v}}
\newcommand{\FV}{\overrightsmallarrow{v}}
\newcommand{\BV}{\overleftrightsmallarrow{v}}

\newcommand{\RG}{\overleftsmallarrow{G}}

\newcommand{\Ppi}{\P^\pi}
\newcommand{\RPpi}{\overleftsmallarrow{\Ppi}}
\newcommand{\dpi}{d^\pi}

\newcommand{\rpi}{r^\pi}
\newcommand{\Rrpi}{\overleftsmallarrow{\rpi}}

\newcommand{\RT}{\overleftsmallarrow{\T}}
\newcommand{\BT}{\overleftrightsmallarrow{\T}}
\newcommand{\FPpi}{\overrightsmallarrow{\Ppi}}

\newcommand{\Rdelta}{\overleftsmallarrow{\delta}}
\newcommand{\Bdelta}{\overleftrightsmallarrow{\delta}}

\usepackage{transparent}

\def\Re{\mathbb{R}}

\def\Nat{{\rm I\kern\pIR N}}
\def\argmax{\mathop{\rm arg\,max}}

\newcommand{\EE}[1]{\exptE\left[#1\right]}

\def\A{{\mathcal{A}}}

\def\P{{\mathcal{P}}}

\def\S{{\mathcal{S}}}
\def\T{{\mathcal{T}}}

\def\vec0{{\boldsymbol{0}}}

\newcommand{\beq}{\begin{equation}}
\newcommand{\eeq}{\end{equation}}
\newcommand{\beqa}{\begin{eqnarray}}
\newcommand{\eeqa}{\end{eqnarray}}
\newcommand{\beqan}{\begin{eqnarray*}}
\newcommand{\eeqan}{\end{eqnarray*}}
\newcommand{\ben}{\begin{eqnarray*}}
\newcommand{\een}{\end{eqnarray*}}

\renewcommand{\EE}[2]{\mathbb{E}_{#1\!\!}\left[#2\right]}

\begin{document}

\maketitle

\begin{abstract}
In this paper, we introduce a fresh perspective on the challenges of credit assignment and policy evaluation.
First, we delve into the nuances of eligibility traces and explore instances where their updates may result in unexpected credit assignment to preceding states. 
From this investigation emerges the concept of a novel value function, which we refer to as the \emph{bidirectional value function}.
Unlike traditional state value functions, bidirectional value functions account for both future expected returns (rewards anticipated from the current state onward) and past expected returns (cumulative rewards from the episode's start to the present).
We derive principled update equations to learn this value function and, through experimentation, demonstrate its efficacy in enhancing the process of policy evaluation. In particular, our results indicate that the proposed learning approach can, in certain challenging contexts, perform policy evaluation more rapidly than TD($\lambda$)---a method that learns forward value functions, $v^\pi$, \emph{directly}. Overall, our findings present a new perspective on eligibility traces and potential advantages associated with the novel value function it inspires, especially for policy evaluation.
\end{abstract}

\section{Introduction}\label{sec:introduction}

Reinforcement Learning (RL) offers a robust framework for tackling complex sequential decision-making problems. The growing relevance of RL in diverse applications—--from controlling nuclear reactors~\citep{radaideh2021physics,park2022control} to guiding atmospheric balloons~\citep{bellemare2020autonomous} and optimizing data centers~\citep{li2019transforming}—--
underscores the need for more efficient solutions. Central to these solutions is addressing the \textit{credit assignment problem}, which involves determining which actions most contributed to a particular outcome---a key challenge in sequential decision-making. The outcome of actions in RL may be delayed, posing a challenge in correctly attributing credit to the most relevant prior states and actions. This often requires a large number of samples or interactions with the environment. In real-world settings, this might be a constraint as system interactions are often risky or costly. Minimizing the sample and computational complexity of existing algorithms not only addresses these challenges but also facilitates the wider adoption of RL in various future applications.


Addressing the credit assignment problem given a temporal sequence of states and actions has been and continues to be an active area of research. A key concept in addressing this challenge is the idea of Temporal Difference (TD) methods~\citep{sutton1984temporal}. 
In the context of TD methods, the \emph{backward view} \cite{sutton2018introduction} offers an intuitive approach under which the agent aims to adjust the value of previous states to align with recent observations. 
As depicted in Figure \ref{fig:td_lambda_intuition}(a), observing a positive outcome at a given time step, $t$, prompts the agent to increase the value estimates of all preceding states. 
This adjustment incorporates a level of discounting, accounting for the diminishing influence of distant preceding states on the observed outcome. This approach relies on the assumption that every prior state should share credit for the resultant outcome. One advantage of the backward view, discussed later, is that it can be efficiently implemented in an online and incremental manner.

\begin{figure}[!t]
    \centering
    \includegraphics[width=0.95\linewidth]{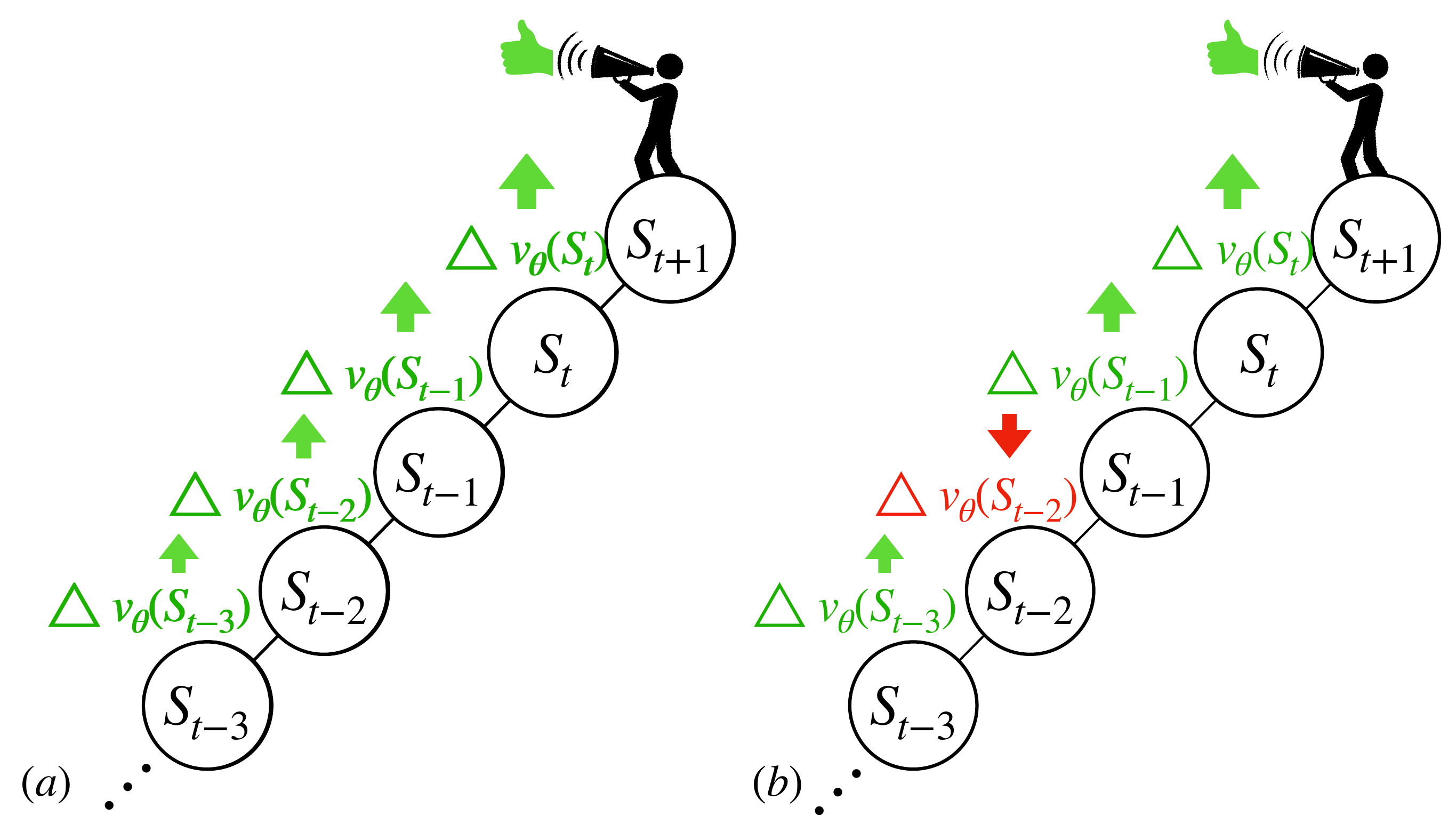}
    \caption{
    Implementation of the backward view via TD($\lambda$), adapted from \citet{sutton2018introduction}. Arrow sizes denote the magnitude of updates: \textbf{(a)} Expected update direction of past state values under the backward view, or TD($\lambda$). \textbf{(b)} 
    Potential misalignment in updates due to reliance on outdated gradient memory for past states (especially when using non-linear function approximation), which deviated from (a).} 
    
    \label{fig:td_lambda_intuition}    
\end{figure}

The TD($\lambda$) method is a widely adopted algorithmic implementation of the backward view~\cite{sutton1984temporal}. Initially introduced for settings involving linear parameterizations of the value function, it has become standard in policy evaluation RL problems. However, as the complexity of RL problems increases with novel applications, there has been a  shift towards adopting neural networks as function approximators, primarily due to their flexibility to represent many functions. 
Yet, this shift is not without challenges. In our work, we underscore one particular issue: when deploying TD($\lambda$) with nonlinear function approximators, particular settings might cause it to update the value of previous states in ways inconsistent with the standard expectations regarding its behavior,  that run counter to the core intuition underlying the backward view (Figure \ref{fig:td_lambda_intuition}\textbf{(a)}). Figure \ref{fig:learning} illustrates such a scenario, where after observing a \textit{positive} outcome, the value of some prior states are \emph{decreased} rather than increased. The direction of these updates, therefore, is contrary to the intended or expected ones.

In a later section, we delve deeper into why this issue with TD($\lambda$) arises. Intuitively, the problem lies in how TD($\lambda$) relies on an ``eligibility trace vector'': a short-term memory of the gradients of previous states' value functions. This trace vector is essentially a moving average of gradients of previously visited states---and it is used by TD($\lambda$) to update the value of multiple states simultaneously. However, as the agent continuously updates state values online, such average gradients can become outdated.

\textbf{In the policy evaluation setting with non-linear function approximation, gradient memory vector maintained by TD($\lambda$) can become outdated, which can pose challenges, leading to state value updates that are misaligned with the intended behavior of the backward view.} As a result, past states might receive updates that do not align with our intentions or expectations. Importantly, this issue does \textit{not} occur under linear functions. This is because, in the linear setting, trace vectors equate to fixed-state feature vectors.

The contributions of this paper are as follows:  
\begin{itemize}
    \item We present a novel perspective under which eligibility traces may be investigated, highlighting specific scenarios that may lead to unexpected credit assignments to prior states.
    
    \item Stemming from our exploration of eligibility traces, we introduce  \emph{bidirectional value functions}. This novel type of value function captures both future and past expected returns, thus offering a broader perspective than traditional state value functions.
    
    \item We formulate principled update equations tailored for learning bidirectional value functions while emphasizing their applicability in practical scenarios.
    
    \item Through empirical analyses, we illustrate how effectively bidirectional value functions can be used in policy evaluation. Our results suggest that the methods proposed can outperform the traditional TD($\lambda$) technique, especially in settings involving complex non-linear approximators.
\end{itemize}

\section{Background \& Motivation}

\subsection{Notation and introduction to RL}
Reinforcement learning is a framework for modeling sequential decision-making problems where an agent interacts with the environment and learns to improve its decisions based on its previous actions and the rewards it receives. The most common way to model such problems is as a Markov Decision Process (MDP), defined as a tuple  $(\S, \A, P, R, d_0, \gamma)$, where $\S$ is the state space; $\A$ is the action space; $P(S_{t+1}=s' | S_t=s, A_t=a)$ is a transition function describing the probability of transitioning to state $s'$ given that the agent was in state $s$ and executed action $a$; $R(S_t,A_t)$ is a bounded reward function; $d_0(S_0)$ is the starting state distribution; and $\gamma$ is the discount factor. For ease of discussion, we also consider the case where \textit{state features} may be used, e.g., to perform value function approximation. In this case, we assume a domain-dependent function $x:\S \rightarrow \Re^d$ mapping states to corresponding $d$-dimensional feature representations. The agent behaves according to a policy denoted by $\pi$. In particular, while interacting with the environment and observing its current state at time $t$, $S_t$, the agent stochastically selects an action $A_t$ according to its policy $\pi: \S \to \Delta(\A)$, where $\Delta(\A)$ is the probability simplex over the actions; i.e., $A_t \sim \pi(S_t)$.
After executing the action, the agent observes a reward value $R(S_t, A_t)$ and transitions to the next state, $S_{t+1}$. 
The goal of an RL agent is to find the policy $\pi^\star$ that maximizes the expected discounted sum of the rewards generated by it:
\begin{align*}
    \pi^\star \in \argmax_{\pi}\EE{\pi}{\sum_{t=0}^{\infty} \gamma^t R(S_t, A_t)  }.
\end{align*}

\begin{figure*}[t]
    \centering
    \includegraphics[width=\textwidth]{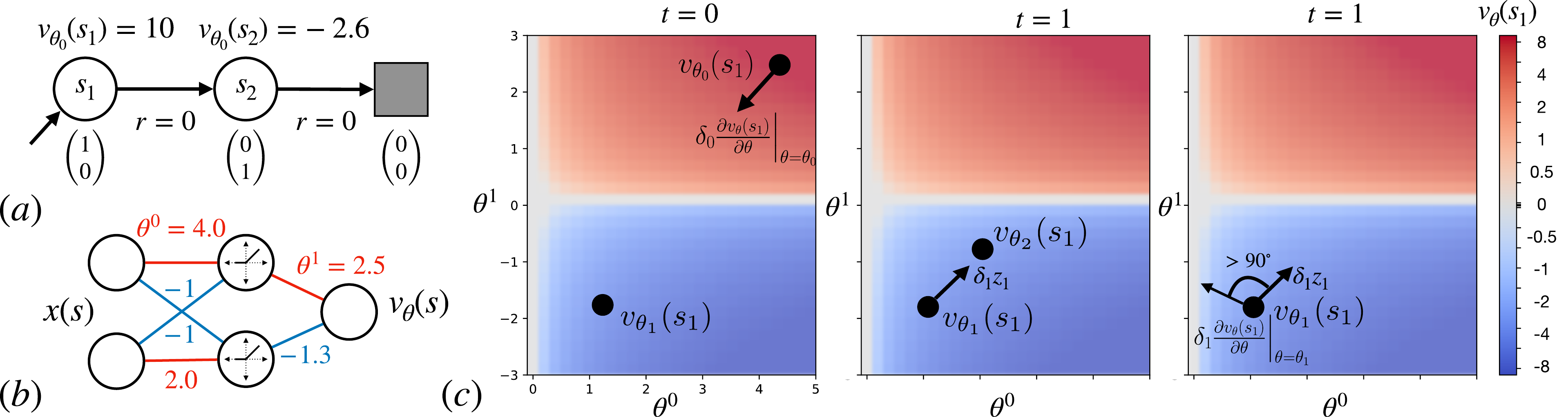}
    \caption{\textbf{(a)} A simple 2-state MDP. States are annotated with corresponding feature representations and values according to the approximator at time $t=0$. \textbf{(b)} Functional form of the value function and respective parameters at $t=0$. \textbf{(c)} On the \textbf{(Left)}, a surface depicting the value of state $s_1$ for different parameter values ($\theta^0$ and $\theta^1$). After an update from $t=0$ to $t=1$, we can see \textbf{(Middle)} the update direction given by TD($\lambda$) at $t=1$ (i.e., $\delta_1 z_1$). On the \textbf{(Right)}, we see that  TD($\lambda$)'s update direction is obtuse (and hence not aligned) to the direction based on the gradient of the \textit{current} value function (i.e., $\delta_1 \tfrac{\partial v_\theta(s_1)}{\partial \theta} |_{\theta = \theta_{1}}$). 
    }
    \label{fig:learning}
\end{figure*}

\subsection{Policy Evaluation}
The search for $\pi^\star$ is often made easier by being able to  
predict the future returns of a given policy---this is known as the \textit{policy evaluation} problem. Return at time $t$ is defined as $G_t \coloneqq R_{t} + \gamma R_{t+1} + \gamma^2 R_{t+2 } + \ldots$, where $R_t \coloneqq R(S_t, A_t)$. We define the value function for a state $s$ as the expected return observed by an agent starting from state $s$ following policy $\pi$, i.e., $v^\pi(s) = \EE{\pi}{G_t | S_t = s}$.

Estimating $v^\pi$ (i.e., performing policy evaluation) is an important subroutine required to perform \textit{policy improvement}. It is often done in settings where the value function is approximated as a parametric function with parameters $\theta$, $v_\theta(s) \approx v^\pi(s)$, where the weights are updated through an iterative process, i.e., $\theta_{t+1} = \theta_t + \triangle \theta_t$.
A common way of determining the update term, $\triangle \theta_t$, is by assuming an update rule of the form 
\begin{align*}
    \triangle \theta_t = \alpha ({\color{black} G_t} - v_{\theta_t}(S_t))   \dfrac{\partial v_\theta(S_t)}{\partial \theta} \Bigr|_{\theta = \theta_{t}},
\end{align*}

\noindent where $\alpha$ is a step-size parameter.
The two simplest algorithms to learn the value function are the Monte Carlo (MC) algorithm and the Temporal Differences (TD) learning algorithm. In the MC algorithm, updates are as follows
\begin{align*}
    \triangle \theta_t = \alpha ( {\color{black} R_t + \gamma G_{t+1}} - v_{\theta_t}(S_t))  \dfrac{\partial v_\theta(S_t)}{\partial \theta} \Bigr|_{\theta = \theta_{t}},
\end{align*}
where determining $G_{t+1}$ requires that the agent wait until the end of an episode.
TD learning algorithms replace the $G_{t+1}$ term with the agent's current approximation or estimate of the return at the next step; i.e., 
\begin{align}\label{eq:td_v}
    \triangle \theta_t = \alpha ({\color{black} R_t + \gamma v_{\theta_t}(S_{t+1}) } - v_{\theta_t} (S_t)) \dfrac{\partial v_\theta(S_t)}{\partial \theta} \Bigr|_{\theta = \theta_{t}},
\end{align}
where $R_t + \gamma v_{\theta_t}(S_{t+1}) - v_{\theta_t} (S_t)$ is known as the TD error and denoted as $\delta_t$. 
An important characteristic of the latter update is its online nature: the agent does not need to wait till the episode ends, since it does not rely on future variables. Thus, updates can be performed after each time step. 

The family of $\lambda$-return algorithms serves as an intermediary between Monte Carlo (MC) and one-step TD(0) algorithms, using a smoothing parameter, $\lambda$. 

TD($\lambda$), which implements the backward view of the $\lambda$-return algorithm, relies only on historical information and can perform value function updates online. It accomplishes this by maintaining a trace vector $e_t$ (known as the eligibility trace) encoding a short-term memory summarizing the history of the trajectory up to time $t$. This trace vector is then used to assign credit to the previous states visited by the agent based on the currently observed reward. In particular, the update term at time $t$ is 
\begin{align}\label{eq:td_lambda}
    \triangle \theta_t = \alpha \delta_t e_t \text{, where } e_t := \sum_{i=0}^{t}  (\lambda \gamma )^{t - i}\frac{\partial v_\theta(S_i)}{\partial \theta} \Bigr|_{ \textcolor{red}{\theta = \theta_{i}}} . 
\end{align}

Note that, when at state $S_t$, $e_t$ can be computed recursively as the running average of the value function gradient evaluated at the states visited during the current episode:
\begin{align}\label{eq:online_trace}
    e_t = \gamma \lambda e_{t-1} + \dfrac{\partial v_\theta(S_t)}{\partial \theta} \Bigr|_{\theta = \theta_{t}},
\end{align}
where $z_{-1}\!=\!0$; i.e., the trace is set to 0 at the start of episodes. 

Note that eligibility traces, as defined above, can be used to perform credit assignment over a single episode. To perform credit assignment over multiple possible episodes, \citet{hasselt2020expected} introduced an \textit{expected} variant of traces, $z(s)$, corresponding to a type of average trace over all possible trajectories ending at state $s$:
\begin{align}\label{eq:expected_trace_old}
    z(s) \coloneqq \EE{}{e_t | S_t = s}.
\end{align}

 \subsection{Misconception with Eligibility Traces}

In TD($\lambda$), the backward view uses the term $\tfrac{\partial v_\theta(S_i)}{\partial \theta}$ as a mechanism to determine how the value of a previously encountered state, $S_i$, will be updated. For instance, given an observed TD error at time $t$, $\delta_t$, we may adjust the value of the state $S_{t-3}$  by using the stored derivatives from time $t-3$, i.e., $\tfrac{\partial v_\theta(S_{t-3})}{\partial \theta}|_{\color{black}\theta=\theta_{t-3}}$.
 Weight updates in TD($\lambda$)  are implemented by maintaining a moving average of the gradients of the value functions related to previously-encountered states (see Eq.~\eqref{eq:td_lambda}). In particular, this weighted average determines how values of multiple past states will be updated given the currently observed TD error.

Notice, however, that this moving average aggregates information, say, about the gradient of the value of state $S_{t-3}$ \textit{assuming the value function at that time}, i.e.,  it aggregates information about the gradient  $\tfrac{\partial v_\theta(S_{t-3})}{\partial \theta}|_{\color{black}\theta=\theta_{t-3}}$. This gradient, however, represents the direction in which weights should be updated (based on the currently-observed outcome, $\delta_t$) to update the \textit{old, no-longer-in-use} value function, $v_{\theta_{t-3}}$, not the  \textit{current} value function, $v_{\theta_t}$. The correct update direction based on the intuition in Figure \ref{fig:td_lambda_intuition}\textbf{(a)}, by contrast, should be that given by the gradient of the \textit{current} value function: $\tfrac{\partial v_\theta(S_{t-3})}{\partial \theta}|_{\color{red}\theta=\theta_{t}}$. 
This indicates that the direction chosen by TD($\lambda$) to update the value of previously encountered states may not align with the correct direction according to the value function's current parameters, $\theta_t$, due to the use of \textbf{outdated gradients}. Mathematically, this discrepancy can be represented as (for $i>0$):
\begin{align*}
        \Big({\frac{\partial v(S_{t-i})}{\partial \theta} \Bigr|_{ \textcolor{black}{\theta = \theta_{t-i}}}} \Big) ^\top
        \Big( \frac{\partial v(S_{t-i})}{\partial \theta} \Bigr|_{\textcolor{red}{\theta = \theta_t}}\Big) < 0 .
\end{align*}
Figure \ref{fig:learning} depicts a simple example to highlight this problem---i.e., the fact that TD($\lambda$) uses outdated gradients when performing updates. In this figure, we can see that the effective update direction of TD($\lambda$) points in the direction opposite to the intended/correct one. We discuss the learning performance of both types of updates in Appendix \ref{appendix:stale_gradients}.

Conceptually, TD($\lambda$) computes traces by combining previous derivatives of the value function to adjust the value of the state at time $i$, based on the TD error at time $t$, for $i < t$. Notably, this misconception was not present when TD($\lambda$) was introduced~\cite{sutton1984temporal}. In its original form, the update was tailored for linear functions.
In this case, the update is a function of the feature representation of each encountered state $S_i$, since $\tfrac{\partial v_\theta(S_i)}{\partial \theta} = x(S_i)$.
Then, the trace is a moving average of features observed in previous steps; update issues are averted since the feature vector of any given state remains the same, independently of changes to the value function.

Recall that equation~\eqref{eq:td_lambda} corresponds to the TD($\lambda$) update and that it highlights how it uses outdated gradients. A minor adjustment to this equation provides the desired update:
\begin{align} \label{eq:td_lambda_2}
    \triangle  \theta_t &=  \alpha \delta_t \tilde{e}_t \text{, where } \tilde{e}_t = \sum_{i=0}^{t} (\lambda \gamma )^{t - i}\frac{\partial v_\theta(S_i)}{\partial \theta} \Bigr|_{\textcolor{red}{\theta = \theta_{t}}} .
\end{align}
The key distinction is in using the latest/current weights, $\theta_t$ (highlighted in \textcolor{red}{red}), during the trace calculation for prior states. For linear function approximations, both \eqref{eq:td_lambda} and \eqref{eq:td_lambda_2} produce identical updates.

A key advantage of the original update \eqref{eq:td_lambda} is that it can be implemented online (in particular, via Eq.~\eqref{eq:online_trace}). This requires only a constant computational and memory cost per update step. In contrast, if we were to use Eq.~ \eqref{eq:td_lambda_2} to perform online updates, it would require computing the derivative for every state encountered up to time $t$. This makes the complexity of each update directly proportional to the episode's duration thus far. This computation becomes impractical for longer episodes. Furthermore, this approach negates from the principle of incremental updates, as the computational cost per step increases based on episode length.

Let us adapt the expected trace formulation~\eqref{eq:expected_trace_old} to our Eq.~\eqref{eq:td_lambda_2}. We start with an expected trace vector $\tilde{z}(s)\coloneqq \EE{}{\tilde{e}_t | S_t = s}$.\footnote{Note that \citet{hasselt2020expected} do not distinguish between $e_t$ and $\tilde{e}_t$. This difference is often overlooked with traces, and one contribution of our work is to point out this difference.} By substituting the value of $\tilde{e}_t$ into this expression and simplifying it, we obtain:
\begin{align}
    \EE{}{\tilde{e}_t | S_t = s} 
    &= \EE{}{\sum_{i=0}^{t} (\lambda \gamma )^{t - i}\frac{\partial v_\theta(S_i)}{\partial \theta} \Bigr|_{\textcolor{black}{\theta = \theta_{t}}} | S_t = s}\\
    &\underset{(a)}{=} \frac{\partial}{\partial \theta} \EE{}{ \sum_{i=0}^{t} \left((\lambda \gamma )^{t - i}  v_\theta(S_i)  \right)\Bigr| S_t=s} \Bigr|_{\theta= \theta_t}\\
    &\underset{(b)}{=}  \frac{\partial f(\theta, s)}{ \partial \theta}\Bigr|_{\theta = \theta_t}. \label{eq:trace_update_functional}
\end{align}
\noindent where $f(\theta, s) := \EE{}{ \sum_{i=0}^{t} \left((\lambda \gamma )^{t - i}  v_\theta(S_i)  \right)\Bigr| S_t=s} $. Step (a) uses linearity of expectation and gradients, and in (b) we define the inner term as a $f$. Notice that the resulting trace is the gradient of a function defined as the weighted sum of value approximations over different time steps.


\section{Methodology}\label{sec:method} 
In the previous section, we showed that the expected trace update, when applied to Eq.~\eqref{eq:td_lambda_2}, is the gradient of a function composed of the expected discounted sum of the value of states as approximated by $v_\theta$. 
Let us consider what Eq.~\eqref{eq:trace_update_functional} would be at the point of convergence, $\theta^\star$, such that $v_{\theta^\star} = v^\pi$. At convergence, $f(\theta^\star, s)$ is 
\begin{align}
    f(\theta^\star, s) &= \EE{}{ \sum_{i=0}^{t} \left((\lambda \gamma )^{t - i}  v_{\theta^\star}(S_i)  \right)\Bigr| S_t=s} \\
    &= \EE{}{ \sum_{i=0}^{t} \left((\lambda \gamma )^{t - i}  v^\pi(S_i)  \right)\Bigr| S_t=s} .
\end{align}
Expanding the definition of $f$ at $\theta^\star$ leads us to Lemma \ref{eq:lemma1}. Prior to delving into the lemma, let us define another quantity, $\RG_t \coloneqq \sum_{i=1}^{t} (\lambda \gamma)^{i} R_{t-i}$. This represents the discounted sum of rewards collected up to the current time step, where discounting is in the reverse direction with a factor of $\lambda\gamma$.

\begin{lemma}
\label{lem:valuefunc} 
The discounted sum of the value function at the expected sequence of states in trajectories reaching state $s$ at time $t$ is
\begin{align}\label{eq:lemma1}
\scriptsize
     &\EE{}{\sum_{i=0}^{t} (\lambda \gamma )^{t - i} v^\pi(S_i) \Bigr| S_t = s}=\\
     &\quad\quad\quad\frac{1}{1 - \gamma^2\lambda}{\EE{}{ { G_t} +  {\RG_t} - \gamma(\lambda\gamma)^{t+1} {G_0} |S_t =s } } .\\
\end{align}
\end{lemma}

\begin{proof}
    Appendix \ref{apx:lemma1_proof}. 
\end{proof}
 
In the equation above, the first two terms are the future and past discounted returns, as defined earlier. 
The third term, $G_0$, represents the return of the entire trajectory conditioned on the agent visiting state $s$ at time $t$; it decays by a factor proportional to $\lambda\gamma$ as the episode progresses.

Similar to how $v^\pi$ is defined as the expectation of $G_t$, we can define another type of value function (akin to $v^\pi$) representing the expected return observed by the agent up to the current time:
\begin{align}\label{eq:backward_value_function}
    \RV^\pi (s) \coloneqq \,\EE{\pi}{\sum_{i=1}^{t} (\lambda \gamma)^i R_{t-i} | S_t = s}.
\end{align}
We call this value function, $\RV$, the \textit{backward value function},   
with the arrow being used to indicate the temporal direction of the prediction.\footnote{For brevity, we often drop $\pi$ from value functions, using $v_\theta$ to represent the  corresponding approximations of these functions.} This value represents the discounted sum of rewards the agent has received until now, where rewards earlier in the trajectory are more heavily discounted. The discount factor for this value function is $\lambda \gamma$, as opposed to the standard discount function, $\gamma$, used in the forward value function, $\FV$. Henceforth, we use $\FV$, $v^\pi$, and $v$ interchangeably.

Let us now define another value function: the sum of the backward and forward value functions. This function combines the expected return to go and the return to get to a state $s$. Simply put, it is the summation of $\RV$ and $\FV$ at a given state:
\begin{align}\label{eq:bidirectional_value_function}
    \BV (s_t) \coloneqq & \RV(s_t) + \FV(s_t) \\
    =&  \EE{}{(\sum_{i=1}^{t} (\gamma \lambda)^i R_{t-i} + \sum_{i=0}^{\infty} \gamma ^i R_{t+i}) | S_t = s_t}.
\end{align}
We refer to this value function as the \textit{bi-directional value function} and denote it as $\BV$ to indicate the directions in which it predicts returns. Notice that we dropped the $\EE{}{\gamma (\lambda\gamma)^{t+1}G_0 | S_t = s}$ term when defining this value function because  
in the limit of $t\rightarrow \infty$, the influence of the starting state decreases, since $\lim_{t\rightarrow\infty}(\lambda\gamma)^{t+1} = 0$.

%
%
Notice that $\BV$ represents (up to a constant scaling factor) the total discounted return received by the agent \textit{throughout an entire trajectory} and passing through a specific state.

In this work, we wish to further investigate the properties of the value functions $\BV$ and $\RV$, and whether learning $\BV$ and $\RV$ can facilitate the identification 
 of the forward value function, $v^\pi$. More precisely, in the next sections we: 
\begin{itemize}
    \item Investigate formal properties of these value functions in terms of learnability and convergence;
    \item Design principled and practical learning rules based on online stochastic sampling;
    \item Evaluate whether learning $\BV^\pi$ may result in a more efficient policy evaluation process, particularly in settings where standard methods may struggle.
\end{itemize}

\section{Bellman Equations and Theory} \label{sec:theory}
In this section, we show that the two newly introduced value functions ($\BV$ and $\RV$) have corresponding variants of Bellman equations. Previous work \cite{zhang2020learninga} has shown that $\RV$ allows for a recursive form (i.e., a Bellman equation), but our work is the first to present a Bellman equation for $\BV$. We also prove that these Bellman equations, when used to define corresponding Bellman operators, are contraction mappings; thus, by applying such Bellman updates we are guaranteed to converge to the corresponding value functions in the tabular setting.


First, we present the standard Bellman equation for the forward value function, $\FV$:
\begin{align*}
    \FV^\pi (s_t) &= \rpi(s_t) + \gamma \sum_{s_{t+1}}\FPpi(s_{t+1} | s_t) \FV^\pi(s_{t+1}),
\end{align*}
wherein we define 
$\rpi(s) = \sum_{a \in \A }\pi(a | s) R(s,a)$ and  $\FPpi(s' | s) = \sum_{a\in \A}\pi(a|s)P(s' | s,a)$. 
This is the standard Bellman equation for the forward value function. Similarly, we can show that the Bellman equation for the \textit{backward} value function can be written as 
\begin{align*}
    \RV^\pi (s_t) &=  \lambda \gamma \Rrpi(s_t) + \lambda \gamma \sum_{s_{t-1}}\RPpi(s_{t-1} | s_t) \RV(s_{t-1}).\\
\end{align*}
The expression above can be proved by applying the recursive definition of $\RV$ on \eqref{eq:backward_value_function},  where $\RPpi(s_{t-1} | s_t)$ and $ \Rrpi(s_t)$ are the backward-looking transition and reward functions. 
Appendix~\ref{appendix:defn_backward_func} presents further details about these definitions and Appendix~\ref{appendix:bellman_reverse} shows the proof/complete derivation of the above Bellman equation.


\begin{theorem}
\label{thm:bellman_bv} 
Given the Bellman equations for $\FV^\pi$ and $\RV^\pi$, the Bellman equation for $\BV^\pi$ is
\begin{align}
    \BV^\pi (s_t) =&  \tfrac{1}{1 + \gamma^2\lambda} \Big( \rpi(s_t)(1 - \gamma^2\lambda) + \\&\gamma \sum_{s_{t+1}}\FPpi(s_{t+1} | s_t) \BV^\pi(s_{t+1}) + \\& \lambda \gamma  \sum_{s_{t-1}}\RPpi(s_{t-1} | s_t) \BV^\pi(s_{t-1}) \Big).
\end{align}
\end{theorem}
\begin{proof}
We provide a proof sketch here. The complete proof is in Appendix \ref{appendix:bidirectional_proof}. To prove this result, we first recall that, by definition, $\BV$ is the sum of $\RV$ and $\FV$. We the expand these terms into their corresponding Bellman equations. Further simplification of terms leads to the above Bellman equation.
\end{proof}
An important point to note in the above equation is that the value of $\BV(s_t)$ bootstraps from the value of the previous state ($\BV(s_{t-1})$), as well as the value of the next state ($\BV(s_{t+1})$), which makes sense since this value function does look in the past as well as future. Another observation regarding this equation is the division by a factor of $\frac{1}{1 + \lambda \gamma^2 }$, which can be viewed as a way to normalize the effect of summing overlapping returns from two bootstrapped state values.

Using the Bellman equations for $\RV_i$ and $\BV_i$, we now define their corresponding Bellman operators, $\RT$ and $\BT$, as:
\begin{align}\label{eq:bellman_operator_reverse}
    \RT(\RV_i(s)) \coloneqq   \lambda \gamma \Rrpi(s) + \lambda \gamma \sum_{s'}\RPpi(s' | s) \RV_i(s'), 
\end{align}
and 
\begin{align}\label{eq:bellman_operator_bidirectional}
    \BT(\BV_i (s')) \coloneqq&  \tfrac{1}{1 + \gamma^2\lambda} ( \rpi(s')(1 - \gamma^2\lambda) + \\& \gamma \sum_{s''}\FPpi(s'' | s') \BV_i(s'') + \\&\lambda \gamma  \sum_{s}\RPpi(s | s') \BV_i(s) ).
\end{align}

\begin{assumption}
\label{assmp:ergodic}
The Markov chain induced by $\pi$ is ergodic. 
\end{assumption}

\begin{theorem}
\label{thm:contraction_rv}
Under Assumption \ref{assmp:ergodic}, the limit $\lim_{t \rightarrow \infty} \EE{}{\RG_t | S_t = s}$ exists and the operator defined in \eqref{eq:bellman_operator_reverse} is a contraction mapping; and hence, repeatedly applying it leads to convergence to $\RV^\pi$.
\end{theorem}
\begin{proof}
The proof is provided in the Appendix \ref{appendix:contraction_reverse}.
\end{proof}

\begin{theorem}
\label{thm:bellman_bv}
Under Assumption \ref{assmp:ergodic}, $\lim_{t \rightarrow \infty} \EE{}{\RG_t + G_t | S_t = s}$ exists and the operator defined in \eqref{eq:bellman_operator_bidirectional} is a contraction mapping; and hence, repeatedly applying it leads to convergence to $\BV^\pi$.
\end{theorem}
\begin{proof}
The proof is provided in the Appendix \ref{appendix:contraction_bidirectional}.
\end{proof}

\subsection{Online and Incremental Update Equations}
In the previous section we introduced the Bellman operators for two value functions and proved that their corresponding operators are contractions.
We would now like to derive an update equation allowing agents to learn such value functions from stochastic samples. Similar to how the TD(0) 
update rule is motivated by its corresponding Bellman operator, we can define similar update rules for $\BV$ and $\RV$. 
Let us parameterize our value functions as follows: $\FV_\theta, \RV_\phi, \BV_\psi$. The TD(0) equivalent of the update equations for the parameters of $\RV_\phi$ and $\BV_\psi$ value functions are presented below.
\begin{align}
    \shortintertext{\textbf{Update for $\psi$ and $\phi$ }:}
    \triangle \psi_t =& \alpha \Bdelta_t \frac{\partial \BV_\psi (S_t)}{\partial \psi} \Bigr|_{\psi = \psi_t},\,\,\, \triangle \phi_t = \alpha \Rdelta_t \frac{\partial \RV_\phi(S_t)}{\partial \phi} \Bigr|_{\phi = \phi_t}.  \label{eq:td_bv_rv} 
\end{align}
\noindent where the corresponding TD error are defined as  $\Bdelta_t \coloneqq \tfrac{1}{1 + \gamma^2\lambda} (R_t (1 - \gamma^2 \lambda) + \gamma \BV_{\psi_t}(S_{t+1}) + \gamma \lambda \BV_{\psi_t}(S_{t-1}))- \BV_{\psi_t} (S_t)) $ and $\Rdelta_t \coloneqq  \lambda \gamma R_{t-1} + \lambda \gamma \RV_{\phi_t} (S_{t-1} ) - \RV_{\phi_t}(S_t)  $.

The above update equations (as is the case with standard TD methods) can be computed online and incrementally; i.e., they incur a fixed computation cost per step and thus allow us to distribute compute evenly throughout an agent's trajectory.

Note that when implementing such updates, we can use a scalar value to store the backward return and obtain an online version of the Monte Carlo update for $\RV$, i.e., $ \RG_t = \lambda \gamma \RG_{t-1} + \lambda \gamma R_{t-1}$,
    leading to the following online update:
\begin{align}
    \triangle \phi_t =& \alpha ( { \color{black} \RG_t } - \RV_{\phi_t}(S_t) ) \frac{\partial \RV_\phi(S_t)}{\partial \phi} \Bigr|_{\phi = \phi_t} \label{eq:mc_rv}.
\end{align}
\noindent where we define $\RG_0 = 0$ and start updating  $\RV$ at time step $t=1$. 
In Appendix \ref{appendix:update_equations} we present several update equation variants that rely on other types of value functions.

\section{Experiments}\label{sec:experiments}
We investigate three questions: \textbf{RQ1}: Can we jointly parameterize all three value functions such that learning each of them individually helps to learn the other two? 
\textbf{RQ2}: Can learning such value functions facilitate/accelerate the process of evaluating forward value functions compared to standard techniques like TD($\lambda$)? \textbf{RQ3}: What is the influence of $\lambda$ on the method's performance? 

\subsection{Parameterization (RQ1)}\label{subsec:parameterization}
We would like to identify a parameterization for our value functions such that training $\BV$ helps learn $\FV$, i.e., the value function we ultimately care about. We can leverage the mathematical property that $\BV = \RV + \FV$ to learn $\FV$ such that these two functions are interdependent and allow for the other to be inferred.

Figure \ref{fig:parameterization} shows a possible way to parameterize the value functions. In this case, we parameterize $\BV$ as the sum of the other two heads of a single-layer neural network. In particular, $\BV= \RV + v$, and so $\theta = \{w^1, w^2\}$, $\phi = \{w^1, w^3\}$, and $\psi=\{w^1, w^2, w^3\} $. We refer to this parameterization as \textbf{BiTD-FR}.  We can similarly fully parameterize the forward value function with all the weights as shown in Appendix \ref{appendix:parameterization} (Figure \ref{fig:parameterization_b}), wherein $v = \BV - \RV$. In this case, the parameterization becomes $\theta = \{w^1, w^2, w^3\}$, $\phi = \{w^1, w^3\}$, and $\psi=\{w^1, w^2\}$;  i.e., we make use of all the weights to parameterize $v$. We refer to this variant as \textbf{BiTD-BiR}. Similarly, for completeness, we define a third parameterization as $\RV = \BV - v$, wherein $\theta = \{w^1, w^2\}$, $\phi = \{ w^1, w^2, w^3\}$, and $\psi = \{w^1, w^3\}$. We call this parameterization \textbf{BiTD-FBi}. 
Note that we have only shown a 1-layer neural network, but  $w^1$ can be replaced with any arbitrarily deep neural network and the parameterization would still remain valid. Our formulation simply imposes more structure in our value function representation. 

Training this network is also straightforward since with a single forward pass we can estimate all three value functions. We can also compute the losses of all three heads with their respective TD/MC updates as shown in \eqref{eq:td_v}, \eqref{eq:td_bv_rv}
and \eqref{eq:mc_rv}. 

\begin{figure}
  \centering
  \begin{minipage}{.5\linewidth}
    \centering
    \includegraphics[width=\linewidth]{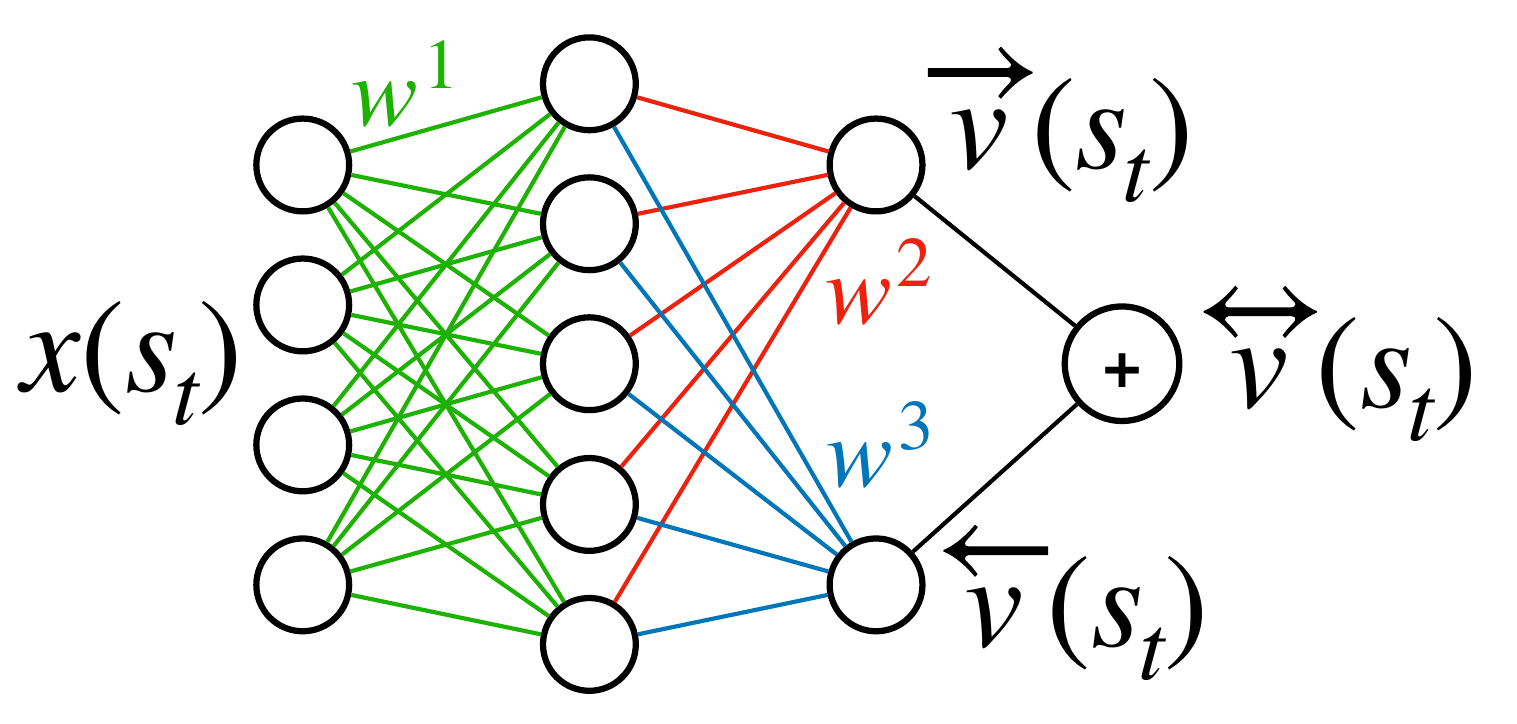}
  \end{minipage}%
  \begin{minipage}{.5\linewidth}
    \caption{Parameterizing the three value functions: We parameterize $\protect\BV$ to be by summation of the other two value functions.}
    \label{fig:parameterization}
  \end{minipage}
\end{figure}



\subsection{Policy Evaluation (RQ2 \& RQ3)}

We study the utility of using these value functions in a standard prediction task. One issue that might arise in value function approximation occurs when trying to approximate non-smooth value functions---i.e., value functions that might change abruptly w.r.t. to the input feature.
 In RL, this implies that the value of spatially similar states may differ vastly. 

\begin{figure}
  \centering
  \includegraphics[width=1.0\linewidth]{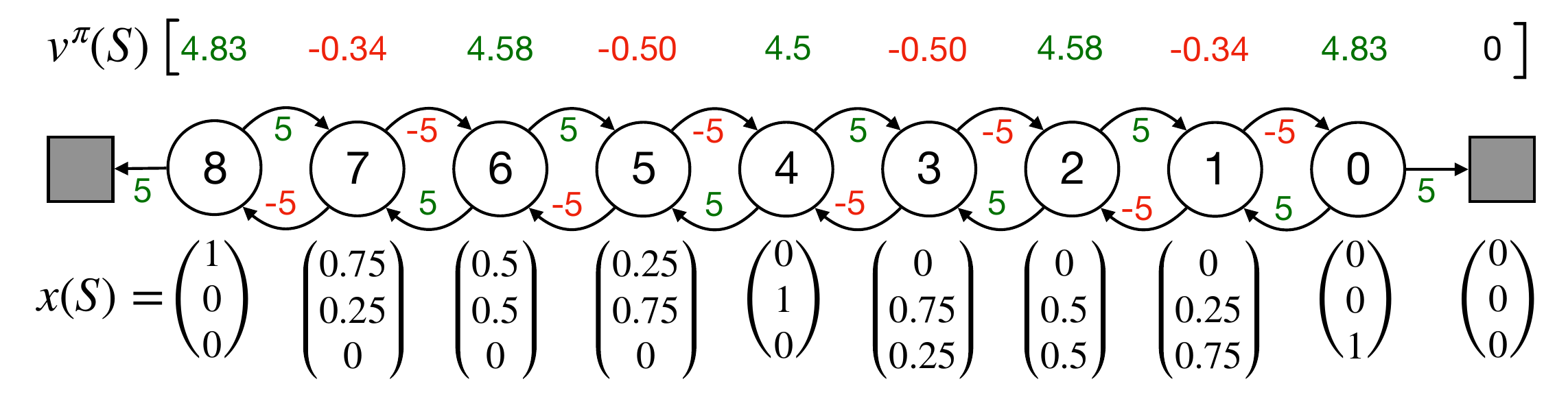}
  \caption{Chain Domain}
  \label{fig:chain}
  \end{figure}

For our prediction problem, we consider a chain domain with 9 states~\cite{sutton2018introduction} as depicted in Figure \ref{fig:chain}. The initial state is drawn from a uniform distribution over the state space. The agent can only take two actions (go left and go right) and the ending states of the chain are terminal. We use a feature representation (motivated by Boyan's chain \cite{boyan2002technical}) such that the values of nearby states are forced to generalize over multiple features---a property commonly observed in continuous-state RL problems. To simulate a highly irregular value function, we define a reward function that fluctuates between -5 and +5 between consecutive states.  

We evaluate each TD learning algorithm (along with the Monte Carlo variant for learning $\RV$) in terms of their ability to approximate the value function of a uniform random policy, $\pi(\texttt{left} | \cdot) = \pi(\texttt{right} | \cdot) =0.5$, under a discount factor of $\gamma = 0.99$. The analytically derived value function is shown in Figure~\ref{fig:chain} alongside the feature representation used for each state. In our experiments, we sweep over multiple values of learning rate ($\alpha$) and $\lambda$. We use the learning rate for all value function heads. Each run corresponds to 50K training/environment steps, and we average the loss function over $100$ seeds. We used a single-layer neural network with 9 units in the hidden layer and ReLU as the non-linearity.

\begin{figure}[!h]
 \centering
    \includegraphics[width=0.70\linewidth]{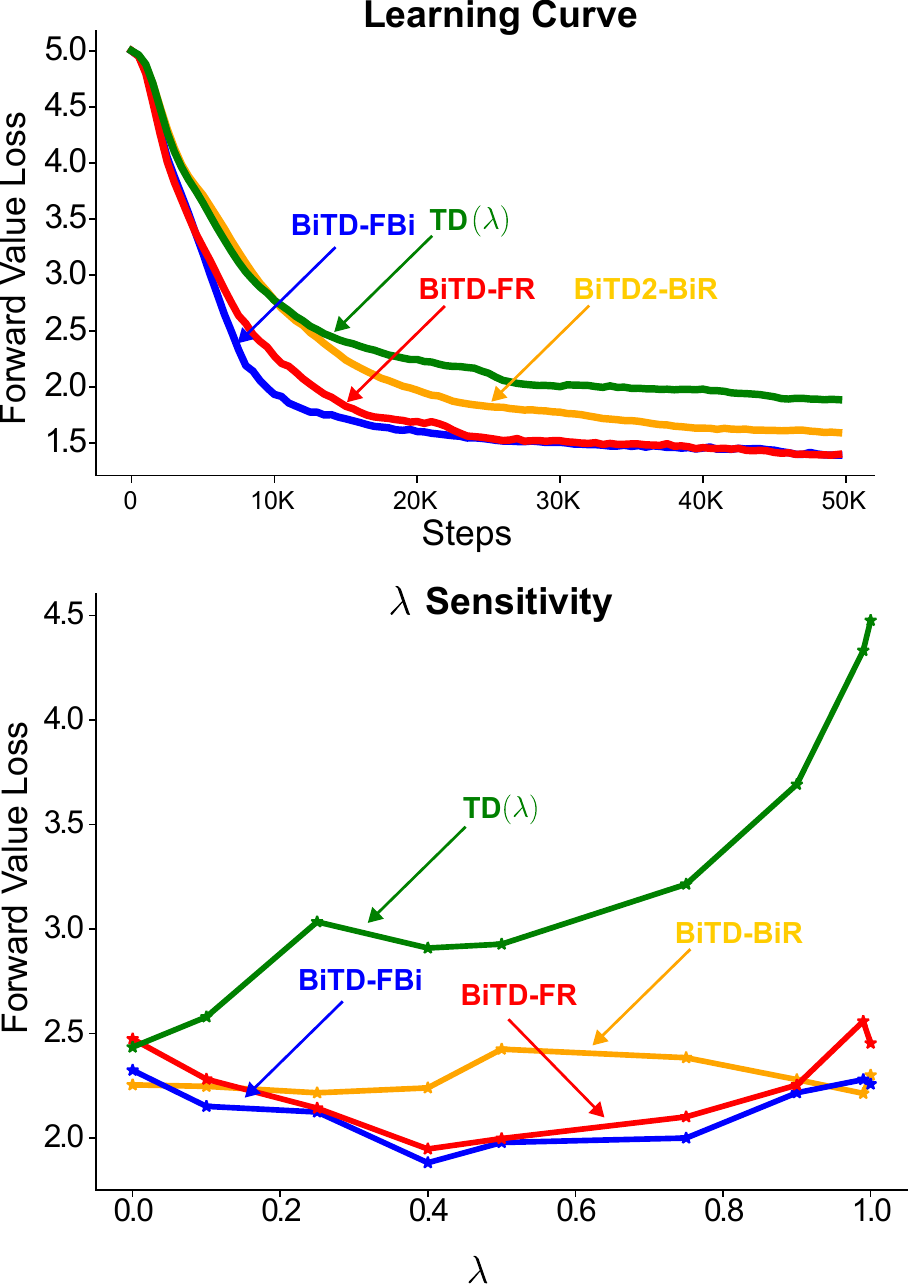}
    \caption{Experimental results for prediction in random chain domain. The $y$ axis shows the MSTDE error of the forward value function. \textbf{(top)} Best performing parameter setting for \textbf{BiTD-FR, BiTD-BiR, BiTD-FBi} and standard \textbf{TD($\lambda$)}. 
    \textbf{(bottom)} We compare all \textbf{BiTD} variants and TD($\lambda$) for different values of $\lambda$; notice that any values $\lambda > 0$ are detrimental to TD($\lambda$)'s performance, but can aid in performing policy evaluation using the proposed framework.}
    \label{fig:experiment}
\end{figure}
Figure \ref{fig:experiment} depicts the results of this experiment with hyperparameters optimized for Area Under Curve (AUC). \textbf{(RQ2)} From Fig.~\ref{fig:experiment}\textbf{(top)} we can see how all variants of \textbf{BiTD} achieve a lower MSTDE loss  than TD($\lambda$) for a given number of samples. \textbf{(RQ3)} To investigate the role of $\lambda$ in the efficiency of policy evaluation, we analyze Fig.~\ref{fig:experiment}\textbf{(bottom)}. From this figure, we can see that TD($\lambda$)'s performance deteriorates strictly as the value of $\lambda$ increases, and that it performs best for $\lambda = 0$. We also notice that \textbf{BiTD-FR} performs similarly to TD for $\lambda = 0$, but its performance is better for intermediate values of $\lambda$ (with the best performance being for $\lambda = 0.4$). Furthermore, notice that among different BiTD methods, the ones that directly approximate $v$ (BiTD-FR and BiTD-FBi) seem to perform better than the ones that indirectly approximate it, like BiTD-BiR. Appendix \ref{appendix:parameterization} provides detailed plots with standard error bars (as well as the sensitivity of different methods w.r.t $\alpha$ and $\lambda$) to better understand how $\alpha$ may affect different methods.


\section{Literature Review}


The successful application of eligibility traces has been historically associated with linear function approximation \citep{sutton2018introduction}. The update rules for eligibility traces, explicitly designed for linear value functions, encounter ambiguities when extended to nonlinear settings. The fact that the RL community started to use non-linear function approximations more often (due to the rise of deep RL) led to the wider use of experience replay, making eligibility traces hard to properly deploy. 
Nonetheless, several works have tried to adapt traces for deep RL \cite{tesaurio1992practical,elfwing2018sigmoid}. 
Traces have found some utility in methods such as advantage estimation \cite{schulman2017proximal}. One interesting interpretation, proposed by \citet{hasselt2020expected}, applies expected traces to the penultimate layer in neural nets while maintaining the running trace for the remaining network.  Traces have also been modified to be combined with experience replay \citep{daley2018efficient}.


Backward TD learning offers a new perspective to RL by integrating ``hindsight'' into the credit assignment process. Unlike traditional forward-view TD learning, which defines updates based on expected future rewards from present decisions, backward TD works retrospectively, estimating present values from future outcomes. This shift to a ``backward view'' has spurred significant advancements.  \citet{chelu2020forethought} underscored the pivotal roles of both foresight and hindsight in RL, illustrating their combined efficacy in algorithmic enhancement.
\citet{wang2021offline} leveraged backward TD for offline RL, demonstrating its potential in settings where data collection proves challenging. Further, the efficacy of backward TD learning methods in imitation learning tasks was highlighted by \citet{park2022robust}. \citet{zhang2020learninga} reinforced the need for retrospective knowledge in RL, underscoring the significance of the backward TD methods.

One may consider learning $\BV$ and $\RV$ as auxiliary tasks. Unlike standard learning scenarios where auxiliary tasks may not directly align with the value function prediction objective, in our case, learning $\BV$ and $\RV$ results in auxiliary tasks that \textit{directly} complement and align with the primary goal of learning the forward value function.
Auxiliary tasks, when incorporated into RL, often act as catalysts, refining the primary task's learning dynamics. Such tasks span various functionalities---next-state prediction, reward forecasting, representation learning, and policy refinement, to name a few \cite{lin2019adaptive,rafiee2022makes}. The ``Unsupervised Auxiliary Tasks'' framework, introduced by \citet{jaderberg2016reinforcement} demonstrates how auxiliary tasks can enhance feature representations, benefiting primary and auxiliary tasks.

\section{Conclusion and Future Work}
In this work, we unveiled an inconsistency resulting from the combination of eligibility traces and non-linear value function approximators. Through a deeper investigation, we derived a new type of value function---a \emph{bidirectional value function}---and showed principled update rules and convergence guarantees. We also introduced online update equations for stochastic sample-based learning methods. Empirical results suggest that this new value function might surpass traditional eligibility traces in specific settings, for various values of $\lambda$, offering a novel perspective to policy evaluation.

Future directions include, e.g., extending our on-policy algorithms to off-policy prediction. Another promising direction relates to exploring analogous functions but in the context of control rather than prediction. We would also like to investigate the hypothesis that bidirectional value functions (akin to average reward policy evaluation methods, which model complete trajectories) may be a first step towards unifying discounted and average-reward RL settings.

\newpage

\bibliography{reference}

\appendix
\onecolumn
\section*{From Past to Future: Rethinking  Eligibility Traces \\(Supplementary Material)}

\section{Experiment with Stale Gradients}\label{appendix:stale_gradients}

\begin{figure}[ht]
  \vskip 0.2in
  \begin{center}
  \centerline{\includegraphics[width=\columnwidth]{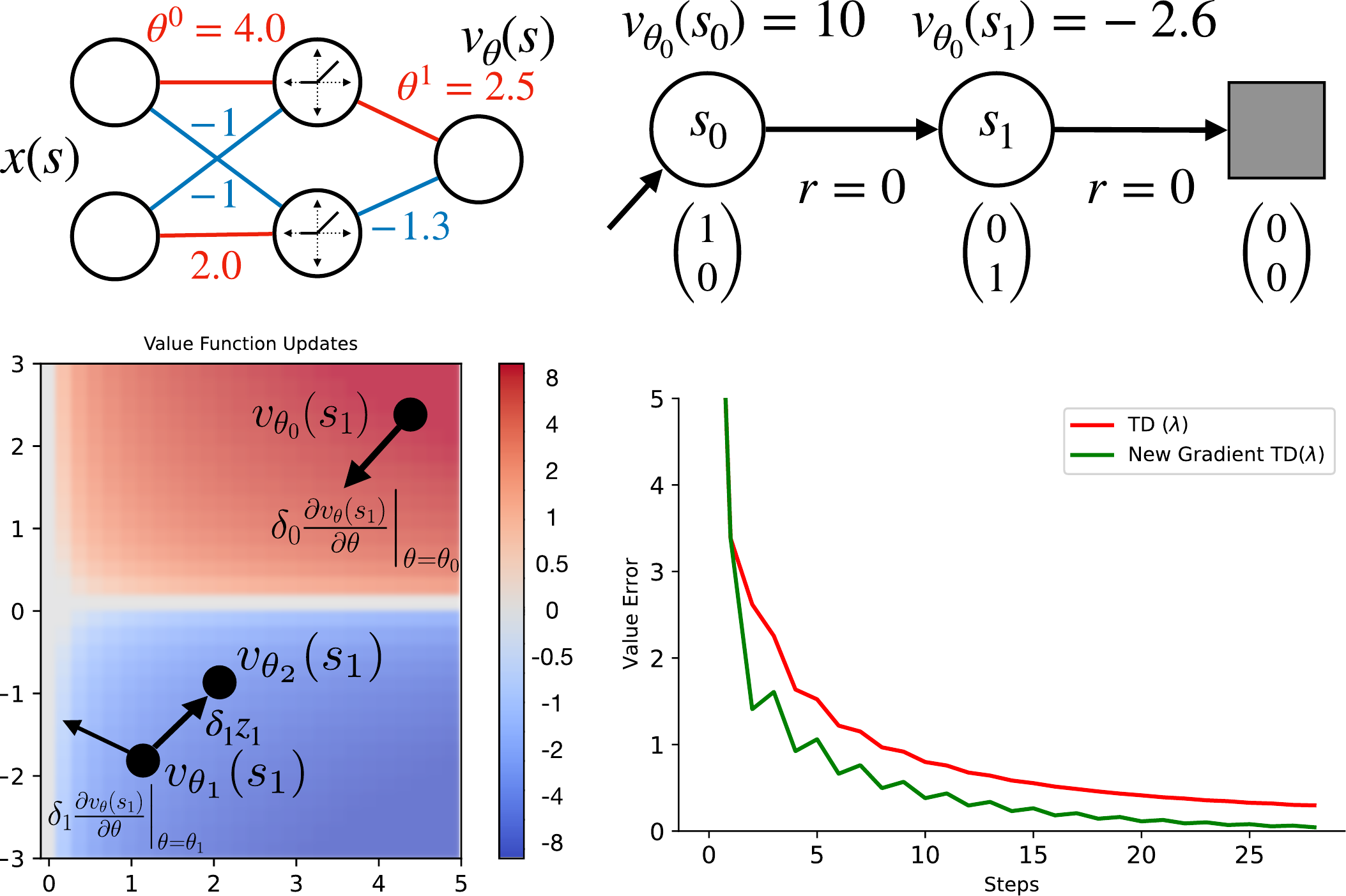}}
  \caption{Example }
  \label{fig:nonlinear_trace}
  \end{center}
  \vskip -0.2in
\end{figure}

Consider a simple example above. We have a 2-state MDP, wherein the agent always starts in state $s_0$ and deterministically transitions to $s_1$ and from there to the terminal state making one episode in its trajectory. At every step the agent receives a reward of 0, hence the true value functions for this MDP is $v(s_0) = v(s_1) = 0$. We use a $\gamma = 0.99, \lambda = 0.95$, and a one-hot feature encoding for these respective states. To simulate the issues with non-linear function approximations, we choose a one-layer neural network as shown in the figure with ReLU activations, wherein we initialize the weights as shown in the figure. We take a close look at the learning schedule of the value of $s_0$, as we see in our value surface plot that, the first update reduces the value of the state (as intended because of the negative $\delta$). But because of the high $\delta$, the value corrects to be negative (as $v_0(s_1)$ is negative as well. Hence, at this point, we are maintaining a trace to update our value functions, and it contains the gradient of the value function wrt $\theta_0$. Now when we move to the update at time $t=1$, we see a positive $\delta$, and hence should rightly update the value of previous states in the positive direction ( At this point both values of states are negative). But because we are using an old trace that still points in the direction to increase the value for $v(s_0)$ based on old weights (and since then some weights have shifted to a point where the gradient points in the opposite direction), this actually leads to a further decrease in the value function. We can see in the graph the ideal direction of the update and the actual direction of the update are at an obtuse angle. 

This also affects the speed of learning as we can see in the learning curve, Note that this method will also converge in the limit, but can slow down the learning at a small scale.  
\section{Definitions for Reverse Transition and Reward Functions}\label{appendix:defn_backward_func}

In this section, we will define the different transition and reward functions that were introduced in Section \ref{sec:theory}. 
In standard literature, we have $\Pr(S_t | S_{t-1}, A_{t-1})$ and the reward function (both looking forward in time). We state the following again for ease of reading, $\FPpi(s'|s) = \sum_{a\in\A} \pi(a|s) P(s'|s,a)$ and $\rpi(s) = \sum_{a\in\A}\pi(a|s)R(s,a)$.

Let us start by defining $\RPpi(s'| s)$ as the probability of being in state $s'$ at time $t-1$, given the agent is in state $s$ at time $t$, i.e.,
\begin{align*}
    \RPpi(s' | s ) &\coloneqq \frac{\Pr(S_{t-1} = s' , S_t = s)}{\Pr(S_t = s)}\\
    & = \dpi(s)^{-1}\dpi(s')  \FPpi( s | s')
\end{align*}

Now lets us define $\Rrpi(s)$, as the expected reward agent received for entering state $s$ at time $t$. Before deriving the expression for this, we define another probability, i.e., 
\begin{align*}
    \Pr(S_{t-1} = s_{t-1}, A_{t-1} = a_{t-1} | S_t = s_t) &= \frac{ \Pr(S_{t-1} = s_{t-1}, A_{t-1} = a_{t-1},S_t =  s_t)}{\Pr(S_t = s_t)}\\
    =& P (s_t | s_{t-1}, a_{t-1}) \dpi (s_{t-1}) \pi(a_{t-1} | s_{t-1})\dpi(s_t)^{-1} \\
    \shortintertext{Hence, let's define a new transition function, which looks backward, i.e.,}\\
    \RPpi(s_{t-1}, a_{t-1} | s_t) \coloneqq & \Pr(S_{t-1} = s_{t-1}, A_{t-1} = a_{t-1} | S_t = s_t) \\
     =& P (s_t | s_{t-1}, a_{t-1}) \dpi (s_{t-1}) \pi(a_{t-1} | s_{t-1})\dpi(s_t)^{-1}. \\
\end{align*}
Hence for $\Rrpi(s_t)$, we get, 
\begin{align*}
    \Rrpi(s) &\coloneqq \sum_{a',s'} \Pr(S_{t-1} = s', A_{t-1} = a' | S_t= s)R(s',a') \\
    &= \sum_{a',s'} \RPpi(s_{t-1}, a_{t-1} | s_t) R(s,a)
\end{align*}

\section{Proof of Lemma \ref{lem:valuefunc}}\label{apx:lemma1_proof}
\begin{align}
    &\EE{}{\sum_{i=0}^t (\lambda\gamma)^{t-i} v^\pi(S_i) | S_t = s} \\
    &= \EE{}{v^\pi(S_t) | S_t = s} + \gamma \lambda \EE{}{v^\pi(S_{t-1})|S_t = s}  + (\gamma \lambda)^2 \EE{}{ v^\pi(S_{t-2}) | S_t = s}  + \ldots  \\ 
    =& \EE{}{G_t | S_t = s} + \gamma \lambda \EE{}{G_{t-1} | S_t = s} + (\gamma \lambda )^2 \EE{}{G_{t-2}|S_t = s}  \ldots  \\
    &= \left( 
        \begin{array}{l}
          (\gamma\lambda)^0\EE{}{({\color{black}R_t} + \gamma {\color{black}R_{t+1}} + \gamma^2 R_{t+2} + \gamma^3 R_{t+3} + \ldots) \Bigr| S_t = s} +  \\ [0.5em]
          (\gamma\lambda)^1\EE{}{({\color{black}R_{t-1}} + \gamma {\color{black}R_t} + \gamma^2 {\color{black}R_{t+1}} + \gamma^3 R_{t+2} + \ldots)\Bigr| S_t = s} + \\ [0.5em]
          (\gamma\lambda)^2\EE{}{(R_{t-2} + \gamma {\color{black}R_{t-1}} + \gamma^2 {\color{black}R_t} + \gamma^3 {\color{black}R_{t+1}} + \ldots)\Bigr| S_t = s} + \ldots \\ [0.5em]
          (\gamma\lambda)^t\EE{}{ (R_{0} + \gamma R_{1} + \gamma^2 R_2 + \gamma^3 R_{4} + \ldots) \Bigr| S_t = s}  \\ [0.5em]
        \end{array} \right)\\
       &= \EE{} {
       \left.
       \begin{array}{l} 
           \;\;\;\;\;\;\;\;\;(\gamma\lambda)^t R_0 ( 1 ) + \\ [0.5em]
           \;\;\;\;\;\;\;\;\;(\gamma\lambda)^{t-1} R_1 ( 1 + \lambda\gamma^2 ) +\\ [0.5em]
           \;\;\;\;\;\;\;\;\; \vdots \\[0.5em]
            \;\;\;\;\;\;\;\;\;(\gamma\lambda)^3 R_{t-3} ( 1 + \lambda \gamma^2 + \lambda^2 \gamma^4 + \ldots + (\lambda\gamma^2)^{t-3}) + \\ [0.5em]
            \;\;\;\;\;\;\;\;\;(\gamma\lambda)^2 R_{t-2} ( 1 + \lambda \gamma^2 + \lambda^2 \gamma^4 + \ldots + (\lambda\gamma^2)^{t-2}) + \\ [0.5em]
            \;\;\;\;\;\;\;\;\;(\gamma\lambda)^1 {\color{black}R_{t-1}}( 1 + \lambda \gamma^2 + \lambda^2 \gamma^4 + \ldots + (\lambda\gamma^2)^{t-1}) + \\ [0.5em]
            \;\;\;\;\;\;\;\;\;(\gamma\lambda)^0 {\color{black}R_t}\quad \left ( 1 + \lambda \gamma^2 + \lambda^2 \gamma^4 + \ldots  + (\lambda \gamma^2)^t \right ) + \\ [0.5em]
            \;\;\;\;\;\;\;\;\;(\gamma)^1 \;\;{\color{black}R_{t+1}} ( 1 + \lambda \gamma^2 + \lambda^2 \gamma^4 + \ldots + (\lambda \gamma^2)^t) + \\ [0.5em]
            \;\;\;\;\;\;\;\;\;(\gamma)^2 \;\;R_{t+2} ( 1 + \lambda \gamma^2 + \lambda^2 \gamma^4 + \ldots + (\lambda \gamma^2)^t) + \ldots \\ [0.5em] 
       \end{array}
       \right\vert S_t = s
       }\\
       &=\left( \sum_{i=0}^{t}(\gamma^2\lambda)^i\right)\EE{}{   {\color{black}R_t} + \gamma  {\color{black}R_{t+1}} + \gamma^2 R_{t+2} \ldots \Bigr| S_t = s } + \\
       & \;\;\;\;\;\;\; \EE{}{ \left.
           \begin{array}{l} 
           (\lambda\gamma)^t (\frac{ 1 - (\lambda\gamma^2)^{t+1-t}}{ 1 - \lambda \gamma^2 }) R_{0} +  \ldots \\ [0.5em]
            (\lambda \gamma )^3 (\frac{1 - (\lambda\gamma^2)^{t-2}}{1 - \lambda \gamma^2})R_{t-3} +  \\ [0.5em]
            (\lambda \gamma )^2 (\frac{1 - (\lambda\gamma^2)^{t-1}}{1 - \lambda \gamma^2})R_{t-2}  + \\ [0.5em]
            (\lambda \gamma )^1 (\frac{1 - (\lambda\gamma^2)^t}{1 - \lambda \gamma^2}){\color{black}R_{t-1}}   \\ [0.5em]
           \end{array}
           \right\vert S_t = s
       }\\
         &=\left(\frac{1 - (\gamma^2\lambda)^{t+1}}{1 - \gamma^2\lambda}\right)\EE{}{   {\color{black}R_t} + \gamma  {\color{black}R_{t+1}} + \gamma^2 R_{t+2} \ldots \Bigr| S_t = s  } + \\
       & \;\;\;\;\;\;\; \EE{}{ \left.
           \begin{array}{l} 
           (\lambda\gamma)^t (\frac{ 1 - (\lambda\gamma^2)^{t+1-t}}{ 1 - \lambda \gamma^2 }) R_{0} +  \ldots \\ [0.5em]
            (\lambda \gamma )^3 (\frac{1 - (\lambda\gamma^2)^{t-2}}{1 - \lambda \gamma^2})R_{t-3} +  \\ [0.5em]
            (\lambda \gamma )^2 (\frac{1 - (\lambda\gamma^2)^{t-1}}{1 - \lambda \gamma^2})R_{t-2}  + \\ [0.5em]
            (\lambda \gamma )^1 (\frac{1 - (\lambda\gamma^2)^t}{1 - \lambda \gamma^2}){\color{black}R_{t-1}}   \\ [0.5em]
           \end{array}
           \right\vert S_t = s
       }\\
       &= \frac{1 - (\gamma^2\lambda)^{t+1}}{1 - \gamma^2\lambda}\EE{}{   \sum_{i=0}^{\infty}\gamma^i R_{t+i} \Bigr| S_t = s} + \EE{}{\sum_{i=1}^{t} (\lambda\gamma)^i \frac{1 - (\lambda\gamma^2)^{t+1-i}}{1-\lambda \gamma^2} R_{t-i} \Bigr| S_t = s } \\
       &= \frac{1}{1 - \gamma^2\lambda} \Big( (1 - (\gamma^2\lambda)^{t+1})\EE{}{   \sum_{i=0}^{\infty}\gamma^i R_{t+i} \Bigr| S_t = s} +\\& \EE{}{\sum_{i=1}^{t} (\lambda\gamma)^i (1 - (\lambda\gamma^2)^{t+1-i}) R_{t-i} \Bigr| S_t = s } \Big)\\
       &=  \frac{1}{1 - \gamma^2\lambda} \Big( (1 - (\gamma^2\lambda)^{t+1})\EE{}{   \sum_{i=0}^{\infty}\gamma^i R_{t+i} \Bigr| S_t = s } +\\& \EE{}{\sum_{i=1}^{t} ((\lambda\gamma)^i  - (\lambda\gamma^2)^{t+1-i} (\lambda\gamma)^i )R_{t-i} \Bigr| S_t = s} \Big)\\
       &=  \frac{1}{1 - \gamma^2\lambda} \Big( (1 - (\gamma^2\lambda)^{t+1})\EE{}{   \sum_{i=0}^{\infty}\gamma^i R_{t+i} \Bigr| S_t = s } + \\& \EE{}{\sum_{i=1}^{t}(\lambda\gamma)^iR_{t-i} \Bigr| S_t = s}   - \EE{}{ \sum_{i=1}^{t}(\lambda\gamma^2)^{t+1-i} (\lambda\gamma)^i )R_{t-i} \Bigr| S_t = s } \Big)\\
       &=  \frac{1}{1 - \gamma^2\lambda} \Big( (1 - (\gamma^2\lambda)^{t+1})\EE{}{   \sum_{i=0}^{\infty}\gamma^i R_{t+i}  \Bigr| S_t = s} + \EE{}{\sum_{i=1}^{t}(\lambda\gamma)^iR_{t-i} \Bigr| S_t = s }   \\&- \EE{}{ \sum_{i=1}^{t}(\lambda^{t+1 - \cancel{i} + \cancel{i} } \gamma^{2t + 2 - \cancel{2i} + \cancel{i}}) R_{t-i}  \Bigr| S_t = s} \Big)\\
       &=  \frac{1}{1 - \gamma^2\lambda} \Big( (1 - (\gamma^2\lambda)^{t+1})\EE{}{   \sum_{i=0}^{\infty}\gamma^i R_{t+i}  \Bigr| S_t = s}  + \EE{}{\sum_{i=1}^{t}(\lambda\gamma)^iR_{t-i} \Bigr| S_t = s}   \\&- \EE{}{ \sum_{i=1}^{t}(\lambda^{t+1} \gamma^{2t + 2 - i}) R_{t-i}  \Bigr| S_t = s } \Big)\\
       &=  \frac{1}{1 - \gamma^2\lambda} \Big( (1 - (\gamma^2\lambda)^{t+1})\EE{}{   \sum_{i=0}^{\infty}\gamma^i R_{t+i} \Bigr| S_t = s} + \EE{}{\sum_{i=1}^{t}(\lambda\gamma)^iR_{t-i} \Bigr| S_t = s}   \\&- \EE{}{ (\lambda \gamma)^{t+1}\sum_{i=1}^{t}( \gamma^{t + 1 - i}) R_{t-i} \Bigr| S_t = s } \Big)\\
       &=  \frac{1}{1 - \gamma^2\lambda} \Big( (1 - (\gamma^2\lambda)^{t+1})\EE{}{   \sum_{i=0}^{\infty}\gamma^i R_{t+i} \Bigr| S_t = s } + \EE{}{\sum_{i=1}^{t}(\lambda\gamma)^iR_{t-i} \Bigr| S_t = s}   \\&- \EE{}{ (\lambda \gamma)^{t+1} \gamma \sum_{i=1}^{t}( \gamma^{t - i}) R_{t-i}  \Bigr| S_t = s} \Big)\\
       &=  \frac{1}{1 - \gamma^2\lambda} \Big( (1 - (\gamma^2\lambda)^{t+1}) \EE{}{   \sum_{i=0}^{\infty}\gamma^i R_{t+i} \Bigr| S_t = s } + \EE{}{\sum_{i=1}^{t}(\lambda\gamma)^iR_{t-i} \Bigr| S_t = s}   - \\& \EE{}{ (\lambda \gamma)^{t+1} \gamma \sum_{i=0}^{t-1}( \gamma^{i}) R_{i} \Bigr| S_t = s } \Big)\\
       &=  \frac{1}{1 - \gamma^2\lambda} \Big( (1 - (\gamma^2\lambda)^{t+1})\EE{}{   \sum_{i=0}^{\infty}\gamma^i R_{t+i} \Bigr| S_t = s} + \EE{}{\sum_{i=1}^{t}(\lambda\gamma)^iR_{t-i} \Bigr| S_t = s}   - \\& (\lambda\gamma)^{t+1}\gamma \EE{}{ \sum_{i=0}^{t-1}( \gamma^{i}) R_{i} \Bigr| S_t = s } \Big)\\
       &=  \frac{1}{1 - \gamma^2\lambda} \Big((1 - (\gamma^2\lambda)^{t+1})\EE{}{   \sum_{i=0}^{\infty}\gamma^i R_{t+i} \Bigr| S_t = s} + \EE{}{\sum_{i=1}^{t}(\lambda\gamma)^iR_{t-i} \Bigr| S_t = s}  \\&  -  (\lambda\gamma)^{t+1}\gamma \Big( \EE{}{ \sum_{i=0}^{\infty}( \gamma^{i}) R_{i} \Bigr| S_t = s } - \EE{}{ \sum_{i=t}^{\infty}( \gamma^{i}) R_{i}  \Bigr| S_t = s } \Big) \Big)\\
        &=  \frac{1}{1 - \gamma^2\lambda} \Big( (1 - (\gamma^2\lambda)^{t+1})\EE{}{   \sum_{i=0}^{\infty}\gamma^i R_{t+i} \Bigr| S_t = s } + \EE{}{\sum_{i=1}^{t}(\lambda\gamma)^iR_{t-i} \Bigr| S_t = s }  \\& -  (\lambda\gamma)^{t+1}\gamma \Big( \EE{}{ \sum_{i=0}^{\infty}( \gamma^{i}) R_{i} \Bigr| S_t = s} - \gamma^t \EE{}{ \sum_{i=0}^{\infty}( \gamma^{i}) R_{t+i} \Bigr| S_t = s } \Big) \Big)\\
        &=  \frac{1}{1 - \gamma^2\lambda} \Big((1 - (\gamma^2\lambda)^{t+1}) \EE{}{   \sum_{i=0}^{\infty}\gamma^i R_{t+i} \Bigr| S_t = s } + \EE{}{\sum_{i=1}^{t}(\lambda\gamma)^iR_{t-i} \Bigr| S_t = s}   \\&-  (\lambda\gamma)^{t+1}\gamma \underbrace{\Big( \EE{}{  G_0 \Bigr| S_t = s } - \gamma^t \EE{}{ G_{t} \Bigr| S_t = s } \Big)}_{\text{t-step return without bootstrapping}} \Big)\\
        &=  \frac{1}{1 - \gamma^2\lambda} \Big( (1 - (\gamma^2\lambda)^{t+1} + (\gamma^2\lambda)^{t+1})\EE{}{ G_t \Bigr| S_t = s} + \EE{}{\RG_t\Bigr| S_t = s}   - \\& (\lambda\gamma)^{t+1}\gamma \EE{}{G_0\Bigr| S_t = s} \Big)\\
        &=  \frac{1}{1 - \gamma^2\lambda} \Big( \EE{}{ G_t \Bigr| S_t = s} + \EE{}{\RG_t\Bigr| S_t = s}   -  (\lambda\gamma)^{t+1}\gamma \EE{}{G_0\Bigr| S_t = s} \Big).\\
\end{align}

\section{Proof for $\protect\RV$ Bellman equation}\label{appendix:bellman_reverse}

Let's start with the definition of the backward value function
\begin{align*}
    \RV^\pi (s) =& \,\EE{\pi}{\sum_{i=1}^{t} (\lambda \gamma)^i R_{t-i} \Big| S_t = s} \\
    =& \, \EE{\pi}{\lambda\gamma R_{t-1} + \sum_{i=2}^{t} (\lambda \gamma)^i R_{t-i} \Big| S_t = s}\\
    =& \, \lambda \gamma \EE{\pi}{R_{t-1} \Big| S_t = s} + \lambda \gamma \EE{\pi}{\sum_{i=1}^{t-1} (\lambda \gamma)^i R_{t-1-i} \Big| S_t = s }\\
    =& \lambda \gamma \sum_{s_{t-1}, a_{t-1}} \Pr(S_{t-1} = s_{t-1},A_{t-1} =  a_{t-1} | S_t = s) r(s_{t-1}, a_{t-1}, s) + \\&
    \lambda \gamma \EE{\pi}{\RG_{t-1} \Big| S_t = s } \\ 
    =& \lambda \gamma \sum_{s_{t-1}, a_{t-1}} \RPpi(s_{t-1}, a_{t-1} |  s) r(s_{t-1}, a_{t-1}, s) + \\& \lambda \gamma \sum_{s_{t-1}} \Pr(S_{t-1} = s_{t-1} | S_t = s) \Pr(\RG_{t-1} | S_{t-1} = s_{t-1}) \RG_{t-1} \\
    =& \lambda \gamma \Rrpi(s) + \lambda \gamma \sum_{s_{t-1}} \RPpi(s_{t-1} | s_t) \EE{\pi}{ \sum_{i=1}^{t-1} (\lambda \gamma)^{i} R_{t-1-i} \Big| S_{t-1} = s_{t-1}}\\
    =& \lambda \gamma \Rrpi(s) + \lambda \gamma \sum_{s_{t-1}} \RPpi(s_{t-1} | s_t) \RV^\pi(s_{t-1})
\end{align*}



\section{Proof Theorem 4.1} \label{appendix:bidirectional_proof}

Considering the limiting case for $\lim t \rightarrow \infty$, we start with the definition of $\BV$ and expand further. 
\begin{align*}
    \BV(s) &= \EE{}{ \sum_{i=1}^{\infty} (\lambda \gamma )R_{t-i} + \sum_{i=0}^\infty R_{t+i} \Bigr| S_t = s}\\
    &= \frac{1 + \gamma^2 \lambda}{1 + \gamma^2 \lambda}\EE{}{ \sum_{i=1}^{\infty} (\lambda \gamma )R_{t-i} + \sum_{i=0}^\infty R_{t+i} \Bigr| S_t = s}\\
    &= \frac{1}{1 + \gamma^2 \lambda}\EE{}{ (1 + \gamma^2\lambda) \Big( \sum_{i=1}^{\infty} (\lambda \gamma )R_{t-i} + \sum_{i=0}^\infty R_{t+i} \Big) \Bigr| S_t = s}\\
       &=\frac{1}{1 + \gamma^2 \lambda} \EE{} {
       \left.
       \begin{array}{l} 
            \;\;\;\;\;\;\;\;\;\vdots \\[0.5em]
           \;\;\;\;\;\;\;\;\;(\gamma\lambda)^2 R_{t-2} (1 + \gamma^2 \lambda)+ \\ [0.5em]
           \;\;\;\;\;\;\;\;\;(\gamma\lambda) R_{t-1} (1 + \gamma^2 \lambda)+  \\ [0.5em]
           \;\;\;\;\;\;\;\;\;\gamma^0\;\; R_{t} (1 + \gamma^2 \lambda)+  \\ [0.5em]
           \;\;\;\;\;\;\;\;\;\gamma \;\;\; R_{t+1} (1 + \gamma^2 \lambda)+  \\ [0.5em]
           \;\;\;\;\;\;\;\;\;\gamma^2\;\;\; R_{t+2} (1 + \gamma^2 \lambda)+  \\ [0.5em]
           \;\;\;\;\;\;\;\;\;\vdots \\[0.5em]
       \end{array}
       \right\vert S_t = s
       }\\
       &=\frac{1}{1 + \gamma^2 \lambda} \EE{} {
       \left.
       \begin{array}{l} 
            \;\;\;\;\;\;\;\;\;\vdots \\[0.5em]
           \;\;\;\;\;\;\;\;\; R_{t-2} ((\gamma\lambda)^2 + \gamma^4 \lambda^3)+ \\ [0.5em]
           \;\;\;\;\;\;\;\;\;R_{t-1} (\gamma\lambda + \gamma^3 \lambda^2)+  \\ [0.5em]
           \;\;\;\;\;\;\;\;\;R_{t} (1 + \gamma^2 \lambda - \gamma^2 \lambda + \gamma^2 \lambda  )+  \\ [0.5em]
           \;\;\;\;\;\;\;\;\;R_{t+1} (\gamma + \gamma^3 \lambda)+  \\ [0.5em]
           \;\;\;\;\;\;\;\;\;R_{t+2} (\gamma^2 + \gamma^4 \lambda)+  \\ [0.5em]
           \;\;\;\;\;\;\;\;\;\vdots \\[0.5em]
       \end{array}
       \right\vert S_t = s
       }\\
       &=\frac{1}{1 + \gamma^2 \lambda}\Big( \EE{} {
       \left.
       \begin{array}{l} 
            \;\;\vdots \\[0.5em]
           \;\; R_{t-2} ((\gamma\lambda)^2 )+ \\ [0.5em]
           \;\;R_{t-1} (\gamma\lambda )+  \\ [0.5em]
           \;\;R_{t} ( \lambda\gamma^2   )+  \\ [0.5em]
           \;\;R_{t+1} (\gamma^3 \lambda )+  \\ [0.5em]
           \;\;R_{t+2} (\gamma^4 \lambda)+  \\ [0.5em]
           \;\;\vdots \\[0.5em]
       \end{array}
       \right\vert S_t = s
       } + 
       \EE{} {
       \left.
       \begin{array}{l} 
            \;\;\vdots \\[0.5em]
            \;\; R_{t-2} (\gamma^4 \lambda^3)+ \\ [0.5em]
            \;\;R_{t-1} ( \gamma^3 \lambda^2)+  \\ [0.5em]
            \;\;R_{t} ( \gamma^2 \lambda )+  \\ [0.5em]
            \;\;R_{t+1} (\gamma)+  \\ [0.5em]
            \;\;R_{t+2} ( \gamma^2)+  \\ [0.5em]
            \;\;\vdots \\[0.5em]
       \end{array}
       \right\vert S_t = s
       } + \EE{}{R_t(1 - \lambda\gamma^2 \Bigr| S_t =s )}\Big)
       \\
       =&\frac{1}{1 + \gamma^2 \lambda}\Big( \lambda\gamma \EE{} {
       \left.
       \begin{array}{l} 
            \;\;\vdots \\[0.5em]
           \;\; R_{t-2} ((\gamma\lambda) )+ \\ [0.5em]
           \;\;R_{t-1} +  \\ [0.5em]
           \;\;R_{t} ( \gamma^1   )+  \\ [0.5em]
           \;\;R_{t+1} (\gamma^2 )+  \\ [0.5em]
           \;\;R_{t+2} (\gamma^3 )+  \\ [0.5em]
           \;\;\vdots \\[0.5em]
       \end{array}
       \right\vert S_t = s
       } + 
       \gamma \EE{} {
       \left.
       \begin{array}{l} 
            \;\;\vdots \\[0.5em]
            \;\; R_{t-2} (\gamma^3 \lambda^3)+ \\ [0.5em]
            \;\;R_{t-1} ( \gamma^2 \lambda^2)+  \\ [0.5em]
            \;\;R_{t} ( \gamma \lambda )+  \\ [0.5em]
            \;\;R_{t+1} +  \\ [0.5em]
            \;\;R_{t+2} ( \gamma^1)+  \\ [0.5em]
            \;\;\vdots \\[0.5em]
       \end{array}
       \right\vert S_t = s
       } + \EE{}{R_t(1 - \lambda\gamma^2 \Bigr| S_t =s )}\Big)
       \\
      =&\frac{1}{1 + \gamma^2 \lambda}\Big( \lambda\gamma \EE{} {
       \left.
       \sum_{i=1}^{\infty} (\lambda\gamma)^iR_{t-1-i} + \sum_{i=0}^{\infty} \gamma^i R_{t-i+i}
       \right\vert S_t = s
       } + 
       \gamma \EE{} {
       \left.
       \sum_{i=i}^{\infty} (\lambda\gamma)^i R_{t+1-i} + \sum_{i=0}^{\infty} \gamma^i R_{t+1+i}
       \right\vert S_t = s
       } +\\ & \EE{}{R_t(1 - \lambda\gamma^2 \Bigr| S_t =s )}\Big)
       \\
       =& \frac{1}{1 + \gamma^2 \lambda}\Big(\lambda\gamma \EE{}{\BV(S_{t-1}) \Bigr| S_t = s} + \gamma \EE{}{\BV(S_{t+1}) \Bigr|S_t = s } + \rpi (s) (1 - \lambda\gamma^2)\Big)\\
       =& \frac{1}{1 + \gamma^2 \lambda}\Big(\lambda\gamma \sum_{s'} \RPpi(s'|s) \BV(s') + \gamma \sum_{s''} \FPpi(s''|s) \BV(s'') + \rpi(s) (1 - \lambda\gamma^2)\Big)
 \end{align*}
 Hence Proved.

\section{Proof Of Theorem 4.3, 4.4}
Proofs for the first statement for Theorem 4.3 follow from the proof of Theorem 1 in \citet{zhang2020learninga}, and following the existence of $\RV$, the summation of the two value functions should also exist. In the following parts, we prove that the operators $\RT, \BT$ are contraction mapping under the $\infty$-norm under the tabular representation, and hence converge to $\RV, \BV $ on repeated applications. 


\subsubsection{Contraction Mapping for $\protect\RT$}\label{appendix:contraction_reverse}
Let value functions,  $\RV_1$ and $\RV_2$ be two value function estimates, and see how the $\RT$ operator behaves under the $\max$ norm. 

\begin{align*}
    || \RT \RV_1 - \RT \RV_2 ||_\infty =& \max_{s_t}  | \RT \RV_1(s_t) - \RT \RV_2(s_t) |\\
    =& \max_{s_t} | \cancel {\lambda \gamma \Rrpi(s_t)} + \lambda \gamma \sum_{s_{t-1}}\RPpi(s_{t-1} | s_t) \RV_1 (s_{t-1}) - ( \cancel{\lambda \gamma \Rrpi(s_t)} + \lambda \gamma \sum_{s_{t-1}}\RPpi(s_{t-1} | s_t) \RV_2(s_{t-1}))  | \\
    =& \max_{s_t} \lambda \gamma |   \sum_{s_{t-1}}\RPpi(s_{t-1} | s_t) \RV_1 (s_{t-1}) - \sum_{s_{t-1}}\RPpi(s_{t-1} | s_t) \RV_2(s_{t-1}) |\\
    =& \max_{s_t} \lambda \gamma |   \sum_{s_{t-1}} \RPpi(s_{t-1} | s_t) (\RV_1 (s_{t-1})  - \RV_2(s_{t-1})) |\\
    & \leq \max_{s_t} \lambda \gamma |   \sum_{s_{t-1}} \RPpi(s_{t-1} | s_t) \max_{s_{t-1}}|\RV_1 (s_{t-1})  - \RV_2(s_{t-1})| |\\
    & \leq \max_{s_t} \lambda \gamma |    \max_{s_{t-1}}|\RV_1 (s_{t-1})  - \RV_2(s_{t-1})|  \sum_{s_{t-1}} \RPpi(s_{t-1} | s_t) |\\
     & \leq \max_{s_t} \lambda \gamma |    \max_{s_{t-1}}|\RV_1 (s_{t-1})  - \RV_2(s_{t-1})|  \times 1 |\\
    & \leq \lambda \gamma || \RV_1 - \RV_2||_\infty 
\end{align*}
\textbf{Note}: In the above we can say that $\sum_{s_{t-1}} \RPpi(s_{t-1} | s_t) = 1$ due to ergodicity in the chain induced by $\pi$, every state is reachable from every other state.
This won't be true only for the case wherein $s_t$ happens to be the starting state. We can take care of this corner case by modifying our MDP such that we have a dummy state from where we always start, and transition to our starting state based on $d_0$. Hence once we include this dummy state into our previous states we can say that $\sum_{s_{t-1}} \RPpi(s_{t-1} | s_t) = 1$. Also, we won't need to consider the dummy state as $s_t$ ever because we don't define a value function on this dummy state (or we can define it as 0). 

\subsubsection{Contraction Mapping for $\protect\BT$}\label{appendix:contraction_bidirectional}

Using similar technique we can start proving contraction mapping for $\BT$ for two value function estimates, i.e.,  $\BV_1$ and $\BV_2$ as follows : 
\begin{align*}
    || \BT \BV_1 - \BT \BV_2 ||_\infty =& \max_{s_t}  | \BT \BV_1(s_t) - \BT \BV_2(s_t) |\\
    =& \max_{s_t} \frac{1}{1 + \lambda \gamma^2}  | ( \gamma \sum_{s_{t+1}}\FPpi(s_{t+1} | s_t) \BV_1(s_{t+1}) + \lambda \gamma  \sum_{s_{t-1}}\RPpi(s_{t-1} | s_t) \BV_1(s_{t-1}) ) \\& - ( \gamma \sum_{s_{t+1}}\FPpi(s_{t+1} | s_t) \BV_2(s_{t+1}) + \lambda \gamma  \sum_{s_{t-1}}\RPpi(s_{t-1} | s_t) \BV_2(s_{t-1}) ) |\\
    =& \max_{s_t} \frac{\gamma}{1 + \lambda \gamma^2}  | (  \sum_{s_{t+1}}\FPpi(s_{t+1} | s_t) \BV_1(s_{t+1}) +  \lambda  \sum_{s_{t-1}}\RPpi(s_{t-1} | s_t) \BV_1(s_{t-1}) ) \\& - (\sum_{s_{t+1}}\FPpi(s_{t+1} | s_t) \BV_2(s_{t+1}) +  \lambda  \sum_{s_{t-1}}\RPpi(s_{t-1} | s_t) \BV_2(s_{t-1}) ) |\\
      =& \max_{s_t} \frac{\gamma}{1 + \lambda \gamma^2}  |  \sum_{s_{t+1}}\FPpi(s_{t+1} | s_t) (\BV_1(s_{t+1}) - \BV_2(s_{t+1})) + \\&   \lambda  \sum_{s_{t-1}}\RPpi(s_{t-1} | s_t) (\BV_1(s_{t-1}) - \BV_2(s_{t-1}) ) |\\
      &\leq \max_{s_t} \frac{\gamma}{1 + \lambda \gamma^2}  ( |  \sum_{s_{t+1}}\FPpi(s_{t+1} | s_t) (\BV_1(s_{t+1}) - \BV_2(s_{t+1})) | + \\&   \lambda | \sum_{s_{t-1}}\RPpi(s_{t-1} | s_t) (\BV_1(s_{t-1}) - \BV_2(s_{t-1}) ) |)\\
      &\leq \max_{s_t} \frac{\gamma}{1 + \lambda \gamma^2}  ( |  \max_{s_{t+1}}(\BV_1(s_{t+1}) - \BV_2(s_{t+1})) \sum_{s_{t+1}}\FPpi(s_{t+1} | s_t)  | + \\&   \lambda  |\max_{s_{t-1}} (\BV_1(s_{t-1}) - \BV_2(s_{t-1}) ) \sum_{s_{t-1}}\RPpi(s_{t-1} | s_t)  |)\\
       &\leq \max_{s_t} \frac{\gamma}{1 + \lambda \gamma^2} (1 + \lambda )|| \BV_1 - \BV_2 ||_\infty\\
\end{align*}

Now the above will be a contraction mapping if $\frac{\gamma (1 + \lambda)}{ 1 + \lambda \gamma^2} < 1$, wherein this  translates  to
\begin{align*}
    \frac{\gamma (1 + \lambda)}{ 1 + \lambda \gamma^2} &< 1\\
    \gamma + \gamma \lambda - \lambda \gamma^2 &< 1\\
    \lambda \gamma (1 - \gamma) & \leq 1 - \gamma\\
    \lambda &\leq \frac{1}{\gamma}.
\end{align*}
The above is always true, because $\gamma < 1$ and $\lambda \leq 1$, and hence $\lambda \leq \frac{1}{\gamma}$ (always).

\subsection{Rate of Convergence}
For the case of $\RV$ we can see that the rate of convergence is $\lambda\gamma$, which is better than the rate of convergence for $v$, which is simply $\gamma$.

As for the case of $\BV$, let us see the conditions under which we would have a better rate of convergence than $\gamma$, i.e., 
\begin{align*}
    \frac{ \cancel{\gamma} (1 + \lambda)}{1 + \lambda \gamma^2} & \leq \cancel{\gamma}\\
    \frac{ (1 + \lambda)}{1 + \lambda \gamma^2} & \leq 1\\
    1 + \lambda & \leq  1 + \lambda \gamma^2\\
    1 & \leq \gamma^2
\end{align*}
As we know that $\gamma < 1$, hence the above condition is never true, and so we have that $\frac{ \cancel{\gamma} (1 + \lambda)}{1 + \lambda \gamma^2} > \cancel{\gamma}$, and hence $\BV$ has a worse rate of convergence.

These rates should be taken with a grain of salt, as these are the worst-case rates and are not representative of what might happen in the actual learning. 

\section{More Update Equations}\label{appendix:update_equations}
In this section, we basically mix and match different value functions to derive various forms of update equations, which primarily differ in how the target is calculated for the corresponding TD errors.

\begin{align*}
    \shortintertext{\textbf{Equations for $\psi$}}
    \triangle \psi_t =& \alpha ({ \color{black} \RV_\phi (s_t) + R_t + \gamma v_{\theta} (s_{t+1})} - \BV_{\psi} (s_t)) \frac{\partial \BV_\psi(s_t)}{\partial \psi} \\
    \triangle \psi_t =& \alpha ( {\color{black}\RV_\phi (s_t) + v_{\theta} (s_{t})} - \BV_{\psi} (s_t)) \frac{\partial \BV_\psi(s_t)}{\partial \psi} \\
    \triangle \psi_t =& \alpha ( {\color{black} \frac{1}{1 + \gamma^2\lambda} (R_t (1 - \gamma^2 \lambda) + \gamma \BV(s_{t+1}) + \gamma \lambda \BV(s_{t+1}))} - \BV_{\psi} (s_t)) \frac{\partial \BV_\psi(s_t)}{\partial \psi} \\
    \shortintertext{\textbf{Equations for $\phi$}}
    \shortintertext{We can easily maintain a scalar of the previous reward values weighted through $\lambda \gamma$ }
    \triangle \phi_t =& \alpha ( {\color{black}G^{\gamma \lambda}_{0:t-1}} - \RV_\phi(s_t) ) \frac{\partial \RV_{\phi}(s_t)}{\partial \phi}\\
    \triangle \phi_t =& \alpha ( {\color{black} \lambda \gamma R_{t-1} + \lambda \gamma \RV_\phi (s_{t-1} )} - \RV_\phi(s_t) ) \frac{\partial \RV_{\phi}(s_t)}{\partial \phi}\\
    \triangle \phi_t =& \alpha ( {\color{black} \BV_\psi(s_t) -  v_\theta (s_{t} )} - \RV_\phi(s_t) ) \frac{\partial \RV_{\phi}(s_t)}{\partial \phi}\\
    \shortintertext{\textbf{Equations for $\theta$}}
    \triangle \theta_t =& \alpha ( {\color{black} R_t + \gamma v_\theta (s_{t+1}) } - v_\theta(s_t) ) \frac{\partial v_\theta (s_t)}{\partial \theta}\\
    \triangle \theta_t =& \alpha ( {\color{black} \BV_\psi(s_t) -  \RV_\phi (s_{t} )} - v_\theta(s_t) ) \frac{\partial v_\theta (s_t)}{\partial \theta}\\
    \triangle \theta_t =& \alpha ( {\color{black} \BV_\psi(s_t) - G^{\gamma \lambda}_{0:t-1})  } - v_\theta(s_t) ) \frac{\partial v_\theta (s_t)}{\partial \theta}\\
\end{align*}

\section{Experiments}
\subsection{Parameterization} \label{appendix:parameterization}
Figure \ref{fig:parameterization_all} provides all the different forms of parameterization present, i.e., BiTD-FR, BiTD-BiR, BiTD-FBi. 
\begin{figure*}[t]
\centering
\begin{subfigure}{.32\textwidth}
  \centering
  \includegraphics[width=0.95\linewidth]{./figures/BiTD.pdf}
 \caption{}
  \label{fig:parameterization_a}
\end{subfigure}%
\begin{subfigure}{.32\textwidth}
  \centering
  \includegraphics[width=0.95\linewidth]{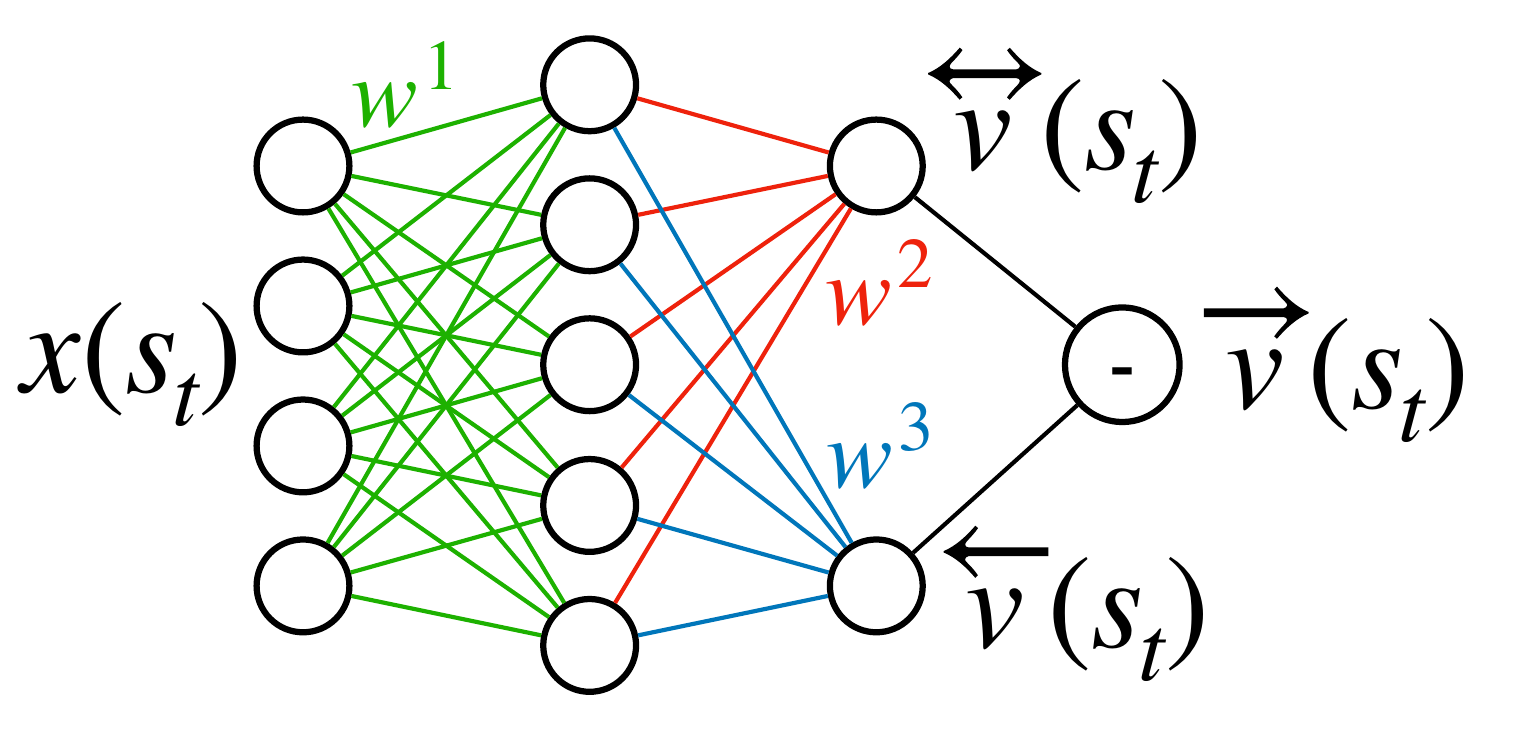}
  \caption{}
  \label{fig:parameterization_b}
\end{subfigure}
\begin{subfigure}{.32\textwidth}
  \centering
  \includegraphics[width=0.95\linewidth]{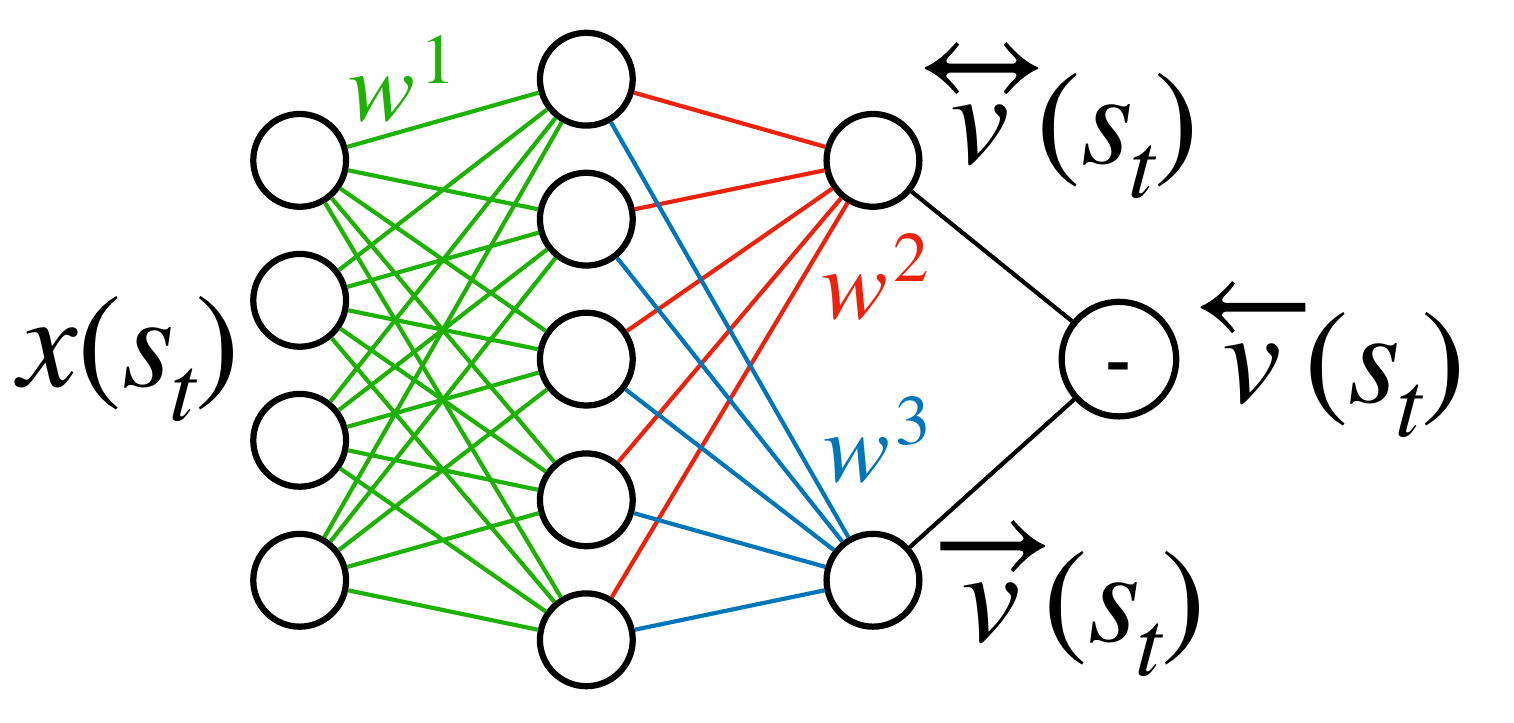}
  \caption{}
  \label{fig:parameterization_c}
\end{subfigure}
\caption{Parameterizing the three value functions as explained in Section \ref{subsec:parameterization} }
\label{fig:parameterization_all}
\end{figure*}
\subsection{More Results}
In Figure \ref{apx:fig:experiment}, we provide the stderr for the learning curves, as well the hyperparameters sensitivity for different methods against $\alpha$. This is to discern if the improved performance in BiTD methods is simply not because of a better-chosen learning rate. To repeat, we used SGD as the optimizer with ReLU and a single hidden layer in the neural network.

\begin{figure}[!h]
  \centering
    \includegraphics[width=1.0\textwidth]{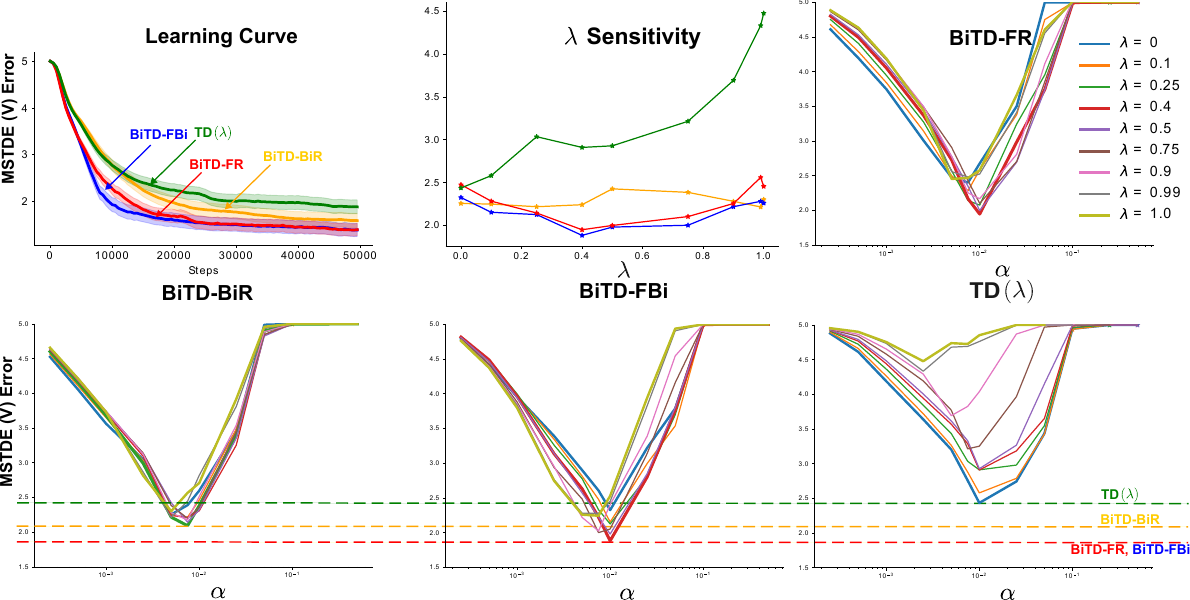}
  \caption{Extended results on the chain domain, all results are averaged over 100 random seeds. All curves have a y-axis for the MSTDE error for the forward value function.   (\textbf{Top Left}) The learning curve with stderr. \textbf{Top Middle} Effect of $\lambda$ on different methods. (repeated from the main paper). (\textbf{Top Right, Bottom}) $\alpha-\lambda$ curves, where x-axis corresponds to different step sizes and different curves correspond to the performance over $\alpha$ for a specific value of $\lambda$. 
  }
  
  \label{apx:fig:experiment}
  \end{figure}

\end{document}